\documentclass[onecolumn]{IEEEtran}

\usepackage[utf8]{inputenc} 
\usepackage[T1]{fontenc}    
\usepackage{hyperref}       
\usepackage{url}            
\usepackage{booktabs}       
\usepackage{nicefrac}       
\usepackage{microtype}      

\usepackage{amsmath,amssymb,amsthm,amsfonts}
\usepackage{graphicx}
\usepackage{algorithm,algorithmic}
\usepackage{xcolor}
\usepackage{epstopdf}

\newtheorem{theorem}{Theorem}
\newtheorem{proposition}{Proposition}
\newtheorem{lemma}{Lemma}

\newtheorem{definition}{Definition}

\newtheorem{remark}{Remark}

\newcommand{\bE}{\mathbb{E}}
\newcommand{\bP}{\mathbb{P}}

\newcommand{\bR}{\mathbb{R}}
\newcommand{\bS}{\mathbb{S}}

\newcommand{\sC}{{\mathcal C}} 
\newcommand{\sD}{{\mathcal D}} 
\newcommand{\sE}{{\mathcal E}} 
 
\newcommand{\sL}{{\mathcal L}} 
 
\newcommand{\sN}{{\mathcal N}}

\newcommand{\sS}{{\mathcal S}}
\newcommand{\sU}{{\mathcal U}} 
\newcommand{\sX}{{\mathcal X}}

\newcommand{\st}{\ | \ }

\newcommand{\ind}[1]{1\left\{#1\right\}} 
\newcommand{\given}{\ | \ }

\newcommand{\brac}[1]{\left[#1\right]}
\newcommand{\set}[1]{\left\{#1\right\}}
\newcommand{\abs}[1]{\left\lvert #1 \right\rvert}
\newcommand{\paren}[1]{\left(#1\right)}
\newcommand{\norm}[1]{\left\|#1\right\|}
\newcommand{\ip}[2]{\left\langle #1,#2 \right\rangle}

\def\ie{{i.e.,~}}
\def\eg{{e.g.,~}}

\newcommand{\algname}[1]{\textnormal{\textsc{#1}}}

\newcommand{\ratio}[1]{r_{#1}}

\newcommand{\rev}[1]{{\color{black}{#1}}}

\newcommand{\jj}[1]{{#1}}
\newcommand{\david}[1]{{#1}}
\newcommand{\laura}[1]{{#1}}

\title{Subspace Clustering using Ensembles of $K$-Subspaces}

\author{{
\sc John Lipor}\\[2pt]
Department of Electrical \& Computer Engineering, Portland State University, \\ Portland, OR, USA \\
{{lipor@pdx.edu}}\\[6pt]
{\sc David Hong\footnote{John Lipor and David Hong made equal contributions to this work.}}\\[2pt]
Wharton Statistics, University of Pennsylvania, Philadelphia, PA, USA\\
{dahong67@wharton.upenn.edu}\\[6pt]
{\sc Yan Shuo Tan}\\[2pt]
Department of Statistics, University of California, Berkeley, Berkeley, CA, USA\\
{yanshuo@berkeley.edu}\\[6pt]
{\sc and}\\[6pt]
{\sc Laura Balzano}\\[2pt]
Department of Electrical Engineering and Computer Science, University of Michigan, \\ Ann Arbor, MI, USA\\
{girasole@umich.edu}
}

\begin{document}

\maketitle

\begin{abstract}
        Subspace clustering is the unsupervised grouping of points lying near a union of low-dimensional linear subspaces. Algorithms based directly on geometric properties of such data tend to either provide poor empirical performance, lack theoretical guarantees, or depend heavily on their initialization. We present a novel geometric approach to the subspace clustering problem that leverages ensembles of the $K$-subspaces (KSS) algorithm via the evidence accumulation clustering framework. Our
        algorithm, referred to as ensemble $K$-subspaces (EKSS), forms a co-association matrix whose $(i,j)$th entry is the number of times points $i$ and $j$ are clustered together by several runs of KSS with random initializations.
        We prove general recovery guarantees for any algorithm that forms an affinity matrix with entries close to a monotonic transformation of pairwise absolute inner products. We then show that a specific instance of EKSS results in an affinity matrix with entries of this form, and hence our proposed algorithm can provably recover subspaces under similar conditions to state-of-the-art algorithms.
    The finding is, to the best of our knowledge, the first recovery guarantee for evidence accumulation clustering and for KSS variants.
We show on synthetic data that our method performs well in the traditionally challenging settings of subspaces with large intersection, subspaces with small principal angles, and noisy data. Finally, we evaluate our algorithm on six common benchmark datasets and show that unlike existing methods, EKSS achieves excellent empirical performance when there are both a small and large number of points per subspace.
\end{abstract}

\section{Introduction}
\label{sec:intro}

In modern computer vision problems such as face recognition \cite{basri2003lambertian} and object tracking
\cite{tron2007benchmark}, researchers have found success applying the union of subspaces (UoS) model, in which data vectors lie near one of several low-rank subspaces. This model can be viewed as a generalization of principal component analysis (PCA) to the case of multiple subspaces, or alternatively a generalization of clustering models where the clusters have low-rank structure. The modeling goal is therefore to simultaneously identify these underlying subspaces and cluster the points according to their nearest subspace. Algorithms designed for this task are called \textit{subspace clustering} algorithms. This topic has received a great deal of attention in recent years
\cite{vidal2016generalized} due to various algorithms' efficacy on real-world problems such as face recognition \cite{georghiades2001from}, handwritten digit recognition \cite{lecun2016mnist}, and motion segmentation \cite{tron2007benchmark}.

Algorithms for subspace clustering can be divided into geometric methods \cite{bradley2000kplane,tseng2000nearest,agarwal2004k,zhang2009median,park2014greedy,heckel2015robust,jalali2017subspace,gitlin2018improving}, which perform clustering by directly utilizing the properties of data lying on a UoS, and self-expressive methods \cite{liu2010robust,lu2012robust,elhamifar2013sparse,you2016scalable,you2016oracle}, which leverage the fact that points lying on a UoS can be efficiently represented by other points in the same subspace. 
For many geometric methods, the inner product between points is a fundamental tool used in algorithm design and theoretical analysis.
In particular, the observation that the inner product between points on the same subspace is often greater than that between points on different subspaces plays a key role. This idea motivates the thresholded subspace clustering (TSC) algorithm \cite{heckel2015robust}, appears in the recovery guarantees of the conic subspace clustering algorithm \cite{jalali2017subspace},
and has been shown to be an effective method of outlier rejection in both robust PCA \cite{rahmani2017coherence} and subspace clustering \cite{gitlin2018improving}. However, despite directly leveraging the UoS structure in the data, geometric methods tend to either exhibit poor empirical performance, lack recovery guarantees, or depend heavily on their initialization.

In this work, we aim to overcome these issues through a set of general recovery guarantees as well as a novel geometric algorithm that achieves state-of-the-art performance across a variety of benchmark datasets. \rev{As our first contribution,} we develop recovery guarantees that match \david{the} state-of-the-art and apply to \emph{any} algorithm that
builds an affinity matrix $A$ with entries close to a monotonic \david{transformation} of pairwise \david{absolute} inner products, i.e., for which
\begin{equation}
    \abs{A_{i,j} - f\left( \abs{\ip{x_{i}}{x_{j}}} \right)} < \tau,
    \label{eq:introObservation}
\end{equation}
where $f$ is a monotonic function, $x_i, x_j$ are data points, and $\tau > 0$ is the maximum deviation.
Such affinity matrices arise in \rev{many modern big data} settings, where only approximate inner products are practically available \rev{or where deviating from inner products may be empirically desirable, but analysis is challenging. An example of the first setting is with dimensionality-reduced data, compressed measurements, or missing data, examples that have become extremely common as we design data-efficient and memory-efficient algorithms for a variety of applications. 
One would also be able to leverage known higher-order structure within the data, such as sparsity structure within each subspace cluster. Our general results would immediately admit theoretical guarantees for an algorithm that deviates from pairwise inner products when leveraging the higher-order structure, as long as the deviation is approximately monotonic.}

\rev{Our second contribution is} the ensemble $K$-subspaces (EKSS) algorithm, which builds its affinity matrix by combining the outputs of many instances of the well-known $K$-subspaces (KSS) algorithm \cite{bradley2000kplane,agarwal2004k} via the \emph{evidence accumulation} clustering framework \cite{fred2005combining}. We show that the affinity matrix obtained from the first iteration of KSS fits the observation model \eqref{eq:introObservation} and consequently enjoys strong
theoretical guarantees. To the best of our knowledge,
these results are the first theoretical guarantees characterizing \david{an} affinity matrix resulting from evidence accumulation, as well as the first recovery guarantees for any variant of the KSS algorithm. Finally, we demonstrate that EKSS achieves excellent empirical performance on \david{several} canonical benchmark datasets.

The remainder of this paper is organized as follows. In Section \ref{sec:related} we define the subspace clustering problem in detail and give an overview of the related work. In Section \ref{sec:ekss} we propose the Ensemble $K$-Subspaces algorithm. Section \ref{sec:theory} contains the theoretical contributions of this paper. We demonstrate the strong empirical performance of EKSS on a variety of datasets in Section \ref{sec:experiments}. Conclusions and future work are \david{described} in Section
\ref{sec:conclusion}.

\section{Problem Formulation \& Related Work}
\label{sec:related}

Consider a collection of points $\sX = \set{x_{1},\dots,x_{N}}$ in $\bR^{D}$ lying near a union of $K$ subspaces $\sS_{1},\dots,\sS_{K}$ having dimensions $d_{1},\dots,d_{K}$. 
Let $X \in \bR^{D \times N}$ denote the matrix whose columns are the \david{points in} $\sX$. The goal of subspace clustering is to label points in the unknown union of $K$ subspaces according to their
nearest subspace. Once the clusters have been obtained, the corresponding subspace bases can be recovered using principal components analysis (PCA).

\subsection{Subspace Clustering}

Most state-of-the-art approaches to subspace clustering rely on a \textit{self-expressive} property of the data, which informally states that points in the UoS model can be most efficiently represented by other points within the same subspace. 
These methods typically use a self-expressive data cost function that is regularized to \david{encourage} efficient representation as follows: 
\begin{eqnarray}
    \label{eq:selfExpress}
    \min_{Z} & \norm{X - XZ}_{F}^{2} + \lambda \norm{Z} \\
    \text{subject to} & \text{diag}(Z) = 0 \nonumber,
\end{eqnarray}
where $\lambda$ balances the regression and penalization terms and $\norm{Z}$ may be the 1-norm as in sparse subspace clustering (SSC) \cite{elhamifar2013sparse}, nuclear norm as in low-rank representation (which omits the constraint on $Z$) \cite{liu2010robust}, or a combination of these and other norms. An affinity/similarity matrix is then obtained as $|Z| + |Z|^{T}$, after which spectral clustering is performed. Other terms are considered in the optimization problem to provide robustness to noise and outliers, and numerous recent papers follow this framework
\cite{soltanolkotabi2012geometric,lu2012robust,vidal2014low,shen2016online}.
For large datasets, solving the above problem may be prohibitive, and algorithms such as \cite{you2016scalable,you2016oracle} employ orthogonal matching pursuit and elastic-net formulations to provide reduced computational complexity and improved connectivity.
These algorithms are typically accompanied by theoretical results that guarantee no false connections (NFC), \ie that \david{points lying in different subspaces have zero affinity.} These guarantees depend on a notion of distance between subspaces called the \emph{subspace affinity}~\eqref{eq:affinity}. Roughly stated, the closer any pair of underlying subspaces is, the more difficult the subspace clustering problem becomes. An excellent overview of these results is given in
\cite{wang2016graph}.

Aside from the self-expressive methods above, a number of geometric approaches have also been considered in the past. Broadly speaking, these methods all determine a set of $q$ ``nearest neighbors'' for each point \david{that} are used to build an affinity matrix, with labels obtained via spectral clustering. An early example of this type of algorithm is the Spectral Local Best-Fit Flats (SLBF) algorithm \cite{zhang2012hybrid}, in which neighbors are selected in terms of Euclidean distance, with the optimal number
of neighbors estimated via the introduced local best-fit heuristic. While this heuristic is theoretically motivated, no clustering guarantees accompany this approach, and its performance on benchmark datasets lags significantly behind that of self-expressive methods. The greedy subspace clustering (GSC) algorithm \cite{park2014greedy} greedily builds subspaces by adding points with largest projection in order to form an affinity matrix, with the number of neighbors fixed. This algorithm
has strong theoretical guarantees, and while its performance is still competitive, it lags behind that of self-expressive methods. Thresholded subspace
clustering (TSC) \cite{heckel2015robust} chooses neighbors based on the largest absolute inner product, and the authors prove that this simple approach obtains correct clustering under assumptions similar to those considered in the analysis of SSC. However, empirical results show that TSC performs poorly on a number of benchmark datasets. Our proposed algorithm possesses the same theoretical guarantees of TSC while also achieving excellent empirical performance.

In contrast to the above methods, the $K$-subspaces (KSS) algorithm \cite{bradley2000kplane,agarwal2004k} seeks to minimize the sum of residuals of points to their assigned subspace, \ie 
\begin{equation}
\label{eq:ksscost}
    \min_{\sC,\sU} \sum_{k=1}^{K} \hspace{1mm} \sum_{i: x_{i} \in c_{k}} \norm{x_{i} - U_{k}U_{k}^{T}x_{i}}_{2}^{2},
\end{equation}
where $\sC = \set{c_{1},\dots,c_{K}}$ denotes the set of estimated clusters and $\sU = \set{U_{1},\dots,U_{K}}$ denotes the corresponding set of orthonormal subspace bases. We claim that this is a ``natural'' choice of objective function for the subspace clustering problem since its value is zero if a perfect UoS fit is obtained. Further, in the case of noiseless data, the optimal solution to \eqref{eq:ksscost} does not depend on how close any pair of subspaces is, indicating that a global
solution to \eqref{eq:ksscost} may be more robust than other objectives to subspaces \david{with} high affinity. However, \eqref{eq:ksscost} was recently shown to be even more difficult to solve than the $K$-means problem in the sense that it is NP-hard to \emph{approximate} within a constant factor \cite{gitlin2018improving} in the worst case. As a result, researchers have turned to the use of alternating algorithms to obtain an approximate solution\david{.} Beginning with an initialization of
$K$ candidate subspace bases, \jj{KSS} \david{alternates between} (i) clustering points by nearest subspace and (ii) obtaining new subspace bases by performing PCA on the points in each cluster. The algorithm is computationally efficient and guaranteed to converge to a
local minimum \cite{bradley2000kplane,tseng2000nearest}\david{, but as} with $K$-means, the KSS output is highly dependent on initialization. It is typically applied by performing many restarts and choosing the result with minimum cost \eqref{eq:ksscost} as the output. This idea was extended to minimize the
median residual (as opposed to mean) in \cite{zhang2009median}, where a heuristic for intelligent initialization is also proposed. In \cite{balzano2012ksubspaces}, the authors use an alternating method based on KSS to perform online subspace clustering in the case of missing data.
In \cite{he2016robust}, the authors propose a novel initialization method based on ideas from \cite{zhang2012hybrid}, and then perform the subspace update step using gradient steps along the Grassmann manifold. While this method is computationally efficient and improves upon the previous performance of KSS, it lacks theoretical guarantees. 
Most recently, the authors of \cite{gitlin2018improving} show that the subspace estimation step in KSS \david{can be} cast as a robust
subspace recovery problem that can be efficiently solved using the Coherence Pursuit (CoP) algorithm \cite{rahmani2017coherence}. The authors motivate the use of CoP by proving that it is capable of rejecting outliers from a UoS and demonstrate that replacing PCA with CoP results in strong empirical performance when there are many points per subspace. However, performance is limited when there are few points per subspace, and the algorithm performance is still highly dependent on the
initialization. Moreover, CoP can be easily integrated into our proposed algorithm to provide improved performance.

Our method is based on the observation that the partially-correct clustering information from each random initialization of KSS can be leveraged using \emph{consensus clustering} in such a way that the consensus is much more informative than even the best single run.
\jj{Unlike the above-mentioned variations on KSS, our proposed approach has cluster recovery guarantees, and its empirical performance is significantly stronger.}

\subsection{Consensus Clustering}

Ensemble methods have been used in the context of general clustering for some time and
fall within the topic of \textit{consensus clustering}, with an overview of the benefits and techniques given in \cite{ghosh2011cluster}. The central idea behind these methods is to obtain many clusterings from a simple base clusterer, such as $K$-means, and then combine the results intelligently. In order to obtain different base clusterings, diversity of some sort must be incorporated. This is
typically done by obtaining bootstrap samples of the data \cite{leisch1999bagged,minaei2004ensembles}, subsampling the data to reduce computational complexity \cite{tumer2008ensemble}, or performing random projections of the data \cite{topchy2005clustering}. Alternatively, the authors of \cite{fred2002evidence,fred2002data} use the randomness in different initializations of $K$-means to obtain diversity. \david{We take this approach here} for subspace clustering. After diversity is achieved, the base
clustering results must be combined. The \textit{evidence accumulation clustering} framework laid out in
\cite{fred2005combining} combines results by voting, \ie creating a co-association matrix $A$ whose $(i,j)$th entry is equal to the number of times two points are clustered together\footnote{\jj{In the context of consensus clustering, we use the terms \emph{affinity matrix} and \emph{co-association matrix} interchangeably.}}. A theoretical framework for this approach is laid out in \cite{bulo2010pairwise}, where the entries of the co-association matrix are modeled as Binomial random variables. This approach is studied further in  
\cite{lourencco2013probabilistic,lourencco2015probabilistic}, in which the clustering problem is solved as a Bregman divergence minimization. These models result in improved clustering performance over previous work but are not accompanied by any theoretical guarantees with regard to the resulting co-association matrix. Further, in our experiments, we did not find the optimization-based approach to perform as well as simply running spectral clustering on the resulting
co-association matrix.

In the remainder of this paper, we apply ideas from consensus clustering to the subspace clustering problem. We describe our ensemble KSS algorithm and its guarantees and demonstrate the algorithm's state-of-the-art performance on \david{both} synthetic and real datasets. 

\section{Ensemble $K$-Subspaces}
\label{sec:ekss}

In this section, we describe our method for subspace clustering using ensembles of the $K$-subspaces algorithm, which we refer to as Ensemble $K$-subspaces (EKSS). 
Our key insight leading to EKSS is the fact that the partially-correct clustering information from each random initialization of KSS can be combined to form a more accurate clustering of the data.
We therefore run several random initializations of KSS and form a co-association
matrix \david{combining their results that} \jj{becomes the affinity matrix used in spectral clustering to obtain cluster labels.}

In more technical detail, our EKSS algorithm proceeds as follows. For each of $b=1,\dots,B$ base clusterings, we obtain an estimated clustering $\sC^{(b)}$ from a single run of KSS with a random initialization of candidate bases. 
The $(i,j)$th entry of the co-association matrix is the number of \david{base clusterings} for which $x_{i}$ and $x_{j}$ are clustered together.
We then threshold the co-association matrix as in
\cite{heckel2015robust} by taking the top
$q$ values from each row/column. Once this thresholded co-association matrix is formed, cluster labels are obtained using spectral clustering. Pseudocode for EKSS is given in Alg. \ref{alg:EKSS},
where \algname{Thresh} sets all but the top $q$ entries in each row/column to zero as in \cite{heckel2015robust} (pseudocode for this procedure is given in Appendix~\ref{app:pseudocode}) and \algname{SpectralClustering} \cite{ng2001spectral} clusters the data points based on the co-association matrix $A$.
Note that the number of candidates $\bar{K}$ and candidate dimension $\bar{d}$ need not match the number $K$ and dimension $d$ of the true underlying subspaces. 
Fig. \ref{fig:progression} shows the progression of the co-association matrix as $B=1, 5, 50$ base clusterings are used for noiseless data from $K = 4$ subspaces of dimension $d = 3$ in an ambient space of dimension $D = 100$ using $\bar{K} = 4$ candidates of dimension $\bar{d}=3$. We discuss the choice of parameters for EKSS in the following sections.

\begin{algorithm}[t]
    \caption{\algname{Ensemble $K$-subspaces (EKSS)}}
    \label{alg:EKSS}
    \begin{algorithmic}[1]
        \STATE \textbf{Input:} $\sX = \set{x_{1},x_{2},\dots,x_{N}} \subset \bR^D$: data, $\bar{K}$: number of candidate subspaces, $\bar{d}$: candidate dimension, $K$: number of output clusters, $q$: threshold parameter, $B$: number of base clusterings, $T$: number of \algname{KSS} iterations
        \STATE \textbf{Output:} $\sC=\set{c_1, \dots,c_K}$: clusters of $\sX$
        \FOR{$b = 1,\dots,B$ (in parallel)}
            \STATE $U_{1},\dots,U_{\bar{K}} \overset{iid}{\sim} \operatorname{Unif}(\operatorname{St}(D,\bar{d}))$ \label{step:candidates} \hfill Draw $\bar{K}$ random subspace bases 
            \STATE $c_k \gets \set{x \in \sX\ :\ \ \forall j \ \norm{U_k^Tx}_2 \geq \norm{U_j^Tx}_2 }$ for $k = 1,\dots,\bar{K}$ \hfill Cluster by projection
            \FOR{$t = 1,\dots,T$ (in sequence)}
            \STATE $U_{k} \gets \algname{PCA}\paren{ c_k, \bar{d} }$ for $k = 1,\dots,\bar{K}$ \label{step:pca} \hfill Estimate subspaces
                \STATE $c_k \gets \set{x \in \sX\ :\ \ \forall j \ \norm{U_k^Tx}_2 \geq \norm{U_j^Tx}_2 }$ for $k = 1,\dots,\bar{K}$ \hfill Cluster by projection
            \ENDFOR
            \STATE $\sC^{(b)} \gets \set{c_{1},\dots,c_{\bar{K}}}$
        \ENDFOR
        \STATE $A_{i,j} \gets \frac{1}{B}\abs{\set{b:x_i,x_j\text{ are co-clustered in }\sC^{(b)}}}$ for $i,j = 1,\dots,N$ \hfill Form co-association matrix
        \STATE $\bar{A} \gets \algname{Thresh}(A,q)$ \hfill Keep top $q$ entries per row/column
        \STATE $\sC \gets$ \algname{SpectralClustering}($\bar{A},K$) \hfill Final Clustering
    \end{algorithmic}
\end{algorithm}

\subsection{Computational Complexity}
\label{sec:complexity}

\begin{figure}[t]
    \centering
    \includegraphics[width=\textwidth]{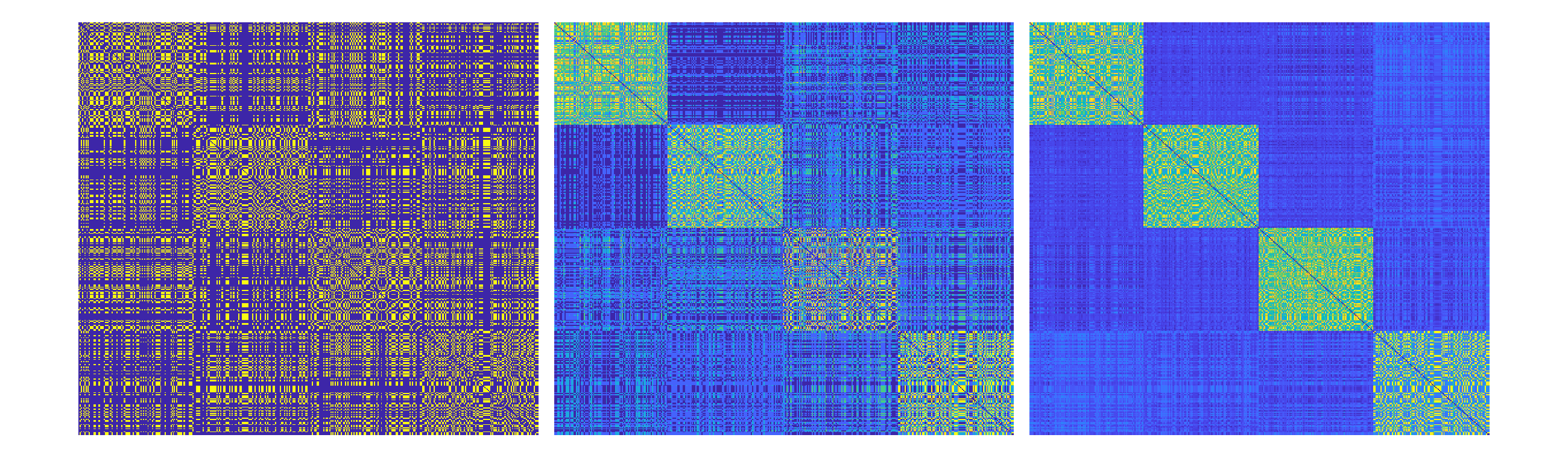}
    \caption{Co-association matrix of EKSS for $B = 1, 5, 50$ base clusterings. Data generation parameters are $D = 100$, $d = 3$, $K = 4$, $N = 400$, and the data is noise-free; the algorithm uses $\bar{K}=4$ candidate subspaces of dimension $\bar{d}=3$ and no thresholding. Resulting clustering errors are 61\%, 25\%, and 0\%.}
    \label{fig:progression}
\end{figure}

Recall the \david{relevant parameters}: $K$ is the number of output clusters, $\bar{K}$ is the number of candidate subspaces in EKSS, $\bar{d}$ is the dimension of those candidates, $N$ is the number of points, $D$ is the ambient dimension, $B$ is the number of KSS base clusterings to combine, and $T$ is the number of iterations within KSS.
To form the co-association matrix, the complexity of EKSS is $O(BT(\bar{K}D^2\bar{d} + \bar{K}D\bar{d}N))$. We run the KSS base clusterings in parallel and use very few iterations, making the functional complexity of EKSS $O(\bar{K}D^2\bar{d} + \bar{K}D\bar{d}N)$, which is competitive with existing methods. In comparison, TSC has complexity $O(DN^2)$ and SSC-ADMM has complexity $O(TN^{3})$, where $T$ is the number of ADMM iterations. 
Note that typically $N > D$ and sometimes much greater. We have not included the cost of spectral clustering, which is $O(KN^2)$. For most modern subspace clustering algorithms (except SSC-ADMM), this dominates the computational complexity as $N$ grows.

\subsection{Parameter Selection}
\label{sec:paramSelection}

EKSS requires a number of input parameters, whose selection we now discuss. As stated in Section \ref{sec:complexity}, we use a small number of KSS iterations, setting $T = 3$ in all experiments. Typically, $B$ should be chosen as large as computation time allows. In our experiments on real data, we choose $B = 1000$. The number of output clusters $K$ is required for all subspace clustering algorithms, and methods such as \david{those} described in \cite{heckel2014neighborhood} can be used
to estimate this value. Hence, the relevant parameters for selection are the candidate parameters $\bar{K}$ and $\bar{d}$ and the thresholding parameter $q$.

When possible, the candidate parameters should be chosen to match the true UoS parameters. In particular, it is advised to set $\bar{K} = K$ and $\bar{d} = d$ when they are known. In practice, a good approximating dimension for the underlying subspace is often known. For example, images of a Lambertian object under varying illumination are known to lie near a subspace with $d = 9$ \cite{basri2003lambertian} and moving objects in video are known to lie near an affine subspace with
$d = 3$ \cite{tomase1992shape}. 
However, as we will show in the following section, our theoretical guarantees hold even if there is model mismatch. Namely, the choice of $\bar{K} = 2$ and $\bar{d} = 1$ still provably yields correct clustering, though this results in a degradation of empirical performance.

The thresholding parameter $q$ can be chosen according to data-driven techniques as in \cite{heckel2014neighborhood}, or following the choice in
\cite{heckel2015robust}. In our experiments on real data, we select $q$ (or the relevant thresholding parameter in the case of SSC) by sweeping over a large range of values and choosing the value corresponding to the lowest clustering error. Note that $q$ is applied to the co-association matrix $A$, and hence the computational complexity of performing model selection is much lower than that of running the entire EKSS algorithm numerous times.

We briefly consider the parameters required by existing algorithms. SSC \cite{elhamifar2013sparse} and EnSC \cite{you2016oracle} both require two parameters to be selected when solving the sparse regression problem \eqref{eq:selfExpress}. SSC also performs thresholding on the affinity matrix, which in our experiments appears critical to performance on real data. See the author code of \cite{elhamifar2013sparse} for details. TSC requires the thresholding parameter
$q$ to be selected. To the best of our knowledge, no principled manner of selecting these parameters has been proposed
in the literature, and we consider this an important issue for future study.

\subsection{Base Clustering Accuracy}
\label{sec:clusteringAccuracy}

\begin{sloppypar}
    A natural heuristic to improve the clustering performance of EKSS is to add larger values to the co-association matrix for base clusterings believed to be more accurate, and smaller values for those believed to be less accurate. Here, we briefly describe one such approach. Note that Step 12 in \algname{EKSS} is equivalent to adding a unit weight to each entry corresponding to co-clustered points, i.e.,
$A \gets \frac{1}{B} \sum_{b = 1}^B A^{(b)} w(b)$,
where $A_{i,j}^{(b)} := \ind{x_i,x_j \text{ are clustered together in }\sC^{(b)}}$ and $w(b) = 1$. 
The key idea is that this weight $w(b)$ can instead be chosen to reflect some estimation of the quality of the $b$th clustering; we propose using the KSS cost function as a measure of clustering quality.
Let $\sC^{(b)} =
\set{c_{1}^{(b)},\dots,c_{K}^{(b)}}$ denote the $b$th base clustering, and let $\sU^{(b)} =
\set{U_{1}^{(b)},\dots,U_{K}^{(b)}}$ denote the set of
subspace bases estimated by performing PCA on the points in the corresponding clusters. The clustering quality can then be measured as
\begin{equation}
    \label{eq:factor}
    w(b) = 1 - \sum_{k = 1}^{K} \sum_{i: x_{i} \in c_{k}^{(b)}} \norm{x_{i} - U_{k}^{(b)} {U_{k}^{(b)}}^{T} x_{i}}_{2}^{2} /\norm{X}_{F}^{2},
\end{equation}
a value between 0 and 1 that decreases as the KSS cost increases. We employ this value of $w(b)$ in all experiments on real data.
\end{sloppypar}



\subsection{Alternative Ensemble Approaches}

As KSS is known to perform poorly in many cases, one may wonder whether better performance can be obtained by applying the evidence accumulation framework to more recent algorithms such as SSC and GSC. We attempted such an approach by subsampling the data to obtain diversity in SSC-OMP \cite{you2016scalable} and EnSC \cite{you2016oracle}. However, the resulting clustering performance did not always surpass that of the base algorithm run on the full dataset. Similar behavior occurred for ensembles of the GSC algorithm
\cite{park2014greedy} as well as the Fast Landmark Subspace Clustering algorithm \cite{wang2015fast}. We also experimented with MKF as a base clustering algorithm but found little or no benefit at a significant increase in computational complexity. Hence, it seems that the success of our proposed approach depends both on the evidence accumulation framework \textit{and} the use of KSS as a base clustering algorithm. Toward this end, we found that EKSS did benefit from the recent CoP-KSS algorithm
\cite{gitlin2018improving} as a base clusterer for larger benchmark datasets, as discussed in Section~\ref{sec:experiments}. The appropriate combination of ensembles of other algorithms is nontrivial and an exciting open topic for future research.

\section{Recovery Guarantees}
\label{sec:theory}

\def\affclose/{angle preserving}
\def\AffClose/{Angle Preserving}
\newcommand{\ang}[2]{#1 \vee #2}


\rev{
In this section, we present theoretical conditions that tie clustering performance to the inner products between points in the dataset. We begin by presenting a general framework that can be applied to any algorithm whose clustering is based on approximate inner products. In particular, we define the notion of an ``\affclose/'' affinity matrix and show that any \affclose/ affinity matrix can be used to obtain clustering with guarantees matching those of state-of-the-art subspace clustering
methods. In Section~\ref{sec:ekssTheory}, we show that EKSS has such an affinity matrix after the first KSS clustering step with high probability, providing the first recovery guarantees for any algorithm based on KSS. This is followed by discussion in Section \ref{sec:resultsdiscuss}. Finally in Section \ref{sec:addguarantees}, we apply our framework to achieve novel results for TSC on dimensionality-reduced data, improving on the results of \cite{heckel2017dimensionality} to show that TSC achieves correct clustering (as opposed to no false connections only) in this case. 
}

We use $N_{max}$ ($N_{min}$) throughout to refer to the maximum (minimum) number of points on any single subspace and $d_{max}$ to refer to the maximum subspace dimension. The proofs of all results in this section \david{are} in Appendix~\ref{app:proofs}.

\subsection{Recovery Guarantees for \AffClose/ Affinity Matrices}
\label{sec:generalTheory}

This section extends the NFC and connectedness guarantees of~\cite{heckel2015robust} to any algorithm that uses \affclose/ affinity matrices. The key idea is that these affinity matrices sufficiently capture the information contained in pairwise angles and obtain good recovery when the angles differentiate the clusters well. Observe that using angles need not be a ``goal'' of such methods; deviating may in fact produce better performance in broader regimes, \eg by incorporating higher order structure. Nevertheless, so long as the relative angles among points are sufficiently captured, the method immediately enjoys the guarantees of this section.

\begin{definition}[\AffClose/]
An affinity matrix $A$ is \emph{$\tau$-\affclose/} for a set of points $\sX$ with respect to a strictly increasing function \david{$f: \bR_+ \to \bR_+$} if
\begin{equation}
\abs{A_{i,j} - f\left( \abs{\ip{x_i}{x_j}} \right)} \leq \tau,
\quad i,j \in [N],
\end{equation}
where we note that $\cos^{-1}\left( \abs{\ip{x_i}{x_j}} \right)$ is the angle between the points $x_i$ and $x_j$.
\end{definition}

Note that $f$ is an arbitrary monotonic transformation \david{that takes small angles (large absolute inner products)} to large affinities and \david{takes} large angles \david{(small absolute inner products)} to small affinities, and $\tau$ quantifies how close the affinity matrix is to such a \david{transformation}. Taking $f(\alpha) = \alpha$ and $\tau = 0$ recovers the absolute inner product. 

To guarantee correct clustering (as opposed to NFC only), it is sufficient to show \david{that the thresholded affinity matrix has both NFC and exactly $K$ connected components} \cite[Appendix A]{heckel2015robust}. We formalize this fact for clarity in the proposition below.
\begin{proposition}[NFC and connectedness give correct clustering\david{~\cite[Equation (15)]{heckel2015robust}}]
    \label{prop:correctClustering}
    Assume that the \david{thresholded} affinity matrix formed by \david{an} algorithm \david{satisfies NFC with probability at least $1-\varepsilon_{1}$ and \david{given NFC} satisfies the connectedness condition with probability at least $1-\varepsilon_{2}$.} Then spectral clustering \david{correctly identifies the components} with probability at least $1 - \varepsilon_{1} - \varepsilon_{2}$.
    \rev{The probabilities here are all with respect to both the randomness in the data and the randomness in the algorithm (if any).}
\end{proposition}

\david{Thus, we study conditions under which NFC and connectedness are guaranteed; conditions for correct clustering follow. In particular, we provide upper bounds on $\tau$ that guarantee NFC (Theorem~\ref{thm:nfc}) and connectedness (Theorem~\ref{thm:cdel}).} \david{The upper bound for NFC is given by a property of the data} that we call the \emph{$q$-angular separation}, defined as follows. We \david{later} bound this quantity in a variety of contexts.

\begin{definition}[Angular Separation]
  \label{def:gap}
  The \emph{$q$-angular separation} $\phi_q$ of the points $\sX = \sX_{1} \cup \dots \cup \sX_{K}$ with respect to a strictly \david{increasing} function \david{$f:\bR_+ \to \bR_+$} is
    \begin{equation}
        \phi_q = \min_{l \in [K], i \in [N_{l}]} \frac{f\left(\abs{ \ip{x_{i}^{(l)}}{x_{\neq i}^{(l)}}}_{[q]}\right) - f\left(\max_{k \neq l, j \in [N_k]} \abs{ \ip{x_{i}^{(l)}}{x_{j}^{(k)}}} \right)}{2},
        \label{eq:taubound_nfc}
    \end{equation}
    where $x_{i}^{(l)}$ denotes the $i$th point of $\sX_{l}$, and $\abs{ \ip{x_{i}^{(l)}}{x_{\neq i}^{(l)}}}_{[q]}$ denotes the $q^{th}$ largest absolute inner product between \david{the} point $x_i^{(l)}$ and other points in subspace $l$. 
\end{definition}

\david{In words, the $q$-angular separation quantifies how far apart the clusters are, as measured by the transformed absolute inner products. When this quantity is positive and large, pairwise angles differentiate clusters well. The following theorem connects this data property to \affclose/ affinity matrices.}

\begin{theorem}[No false connections (NFC)]
    \label{thm:nfc}
    \david{Suppose $\sX = \sX_{1} \cup \dots \cup \sX_{K}$ have $q$-angular separation $\phi_q$ with respect to a strictly increasing function $f$.} Then the $q$-nearest neighbor graph for any $\phi_q$-\affclose/ affinity matrix \david{(with respect to $f$)} has no false connections.
\end{theorem}

Theorem~\ref{thm:nfc} states that sufficiently small deviation $\tau$ guarantees NFC as long as \david{the data has positive $q$-angular separation.}
The next theorem provides an upper bound on $\tau$ that guarantees connectedness within a cluster with high probability \david{given NFC}.
\david{Under NFC, the $q$-nearest neighbors of any point (with respect to the affinity matrix) are in the same subspace, and so the theorem is stated with respect to only points within a single subspace. In particular, we restrict to the $d$-dimensional subspace and so consider the $q$-nearest neighbor graph $\tilde{G}$ for  points $a_1,\dots,a_n$ uniformly distributed on the sphere $\bS^{d-1}$.}

\begin{theorem}[Connectedness]
    \label{thm:cdel}
    Let $a_1,\dots,a_n \in \bR^d$ be i.i.d. uniform on $\bS^{d-1}$\rev{,
    and choose any $\gamma > 1$ for which}
    a spherical cap covering $\gamma \log n /n$ of the area of $\bS^{d-1}$
    \rev{has spherical radius less than $\pi/48$.
    If $q \geq 4(24\pi)^{d-1}\gamma \log n$,
    then with probability at least $1 - 2/(n^{\gamma-1}\gamma \log n)$
    any $C_3$-\affclose/ affinity matrix
    has a connected $q$-nearest neighbor graph,
    where $C_3$ is defined in the proof and
    depends only on $d$, $n$, $\gamma$, and
    the function $f$ with respect to which the affinity matrix is \affclose/.
    Note that the probability here is with respect to $\{a_i\}$.}
\end{theorem}

We now provide explicit high-probability lower bounds on the $q$-angular separation $\phi_{q}$ from \eqref{eq:taubound_nfc} in some important settings relevant to subspace clustering. These results can be used to guarantee NFC by bounding the deviation level $\tau$.
Consider first the case where there is no intersection between any pair of subspaces but there are potentially unobserved entries, \ie missing data.
Lemma~\ref{lem:noIntersection} \david{bounds $\phi_q$ from below} in such a setting; the \david{bound} depends on a variant of the minimum principal angle between subspaces that accounts for missing data. 

\begin{lemma}[\david{Angular separation} for missing data]
    \label{lem:noIntersection}
    Let $\sS_{k}$, $k=1,\dots,K$ be subspaces of dimension $d_1,\dots,d_K$ in $\bR^D$. Let the $N_{k}$ points in $\sX_{k}$ be drawn as $x_{j}^{(k)} = U^{(k)}a_{j}^{(k)}$, where \rev{each $a_{j}^{(k)}$ is independently drawn} uniform on $\bS^{d_{k}-1}$ and $U^{(k)} \in \bR^{D \times d_{k}}$ has (not necessarily orthonormal) columns that form a basis for $\sS_k$.
    In each $x_{j} \in \sX$, up to $s$ \rev{(arbitrarily chosen)} entries are then unobserved, \ie set to zero. Let $\rho \in [0,1)$ be arbitrary and \rev{set $q < N_{min}^{\rho}$. S}uppose that $N_{min} > N_{0}$ and
    \begin{equation}
        \ratio{s} = \frac{ \max_{k,l: k \neq l, \sD: \abs{\sD} \leq 2s} \norm{{U_{\sD}^{(k)}}^{T}U^{(l)}}_{2} }{ \min_{l, \sD: \abs{\sD} \leq 2s, \norm{a} = 1} \norm{{U_{\sD}^{(l)}}^{T}U^{(l)}a}_{2} } < 1,
        \label{eq:affCondMissing}
    \end{equation}
    \rev{where $N_{0}$ here is a numerical constant that depends only on $d_{max}$ and $\rho$,
    and $U_{\sD}^{(l)}$ denotes} the matrix obtained from $U^{(l)}$ by setting the rows indexed by $\sD \subset \set{1,\dots,D}$ to zero.
    Then \rev{the $q$-angular separation of these partially observed points is bounded as} $\phi_{q} > C_1$ with probability at least $1 - \sum_{k=1}^KN_ke^{-c_1(N_{k}-1)}$, where $c_1 > 0$ is a numerical constant that depends on $N_{min}^{\rho}$, and $C_1 >0$ depends only on $\ratio{s}$ and \rev{the function $f$ that the $q$-angular separation is with respect to}. Both $c_1$ and $C_1$ are defined in the proof\rev{,
    and the probability here is with respect to the randomness from the coefficients $\{a_j^{(k)}\}$}.
\end{lemma}
\begin{sloppypar}
To gain insight to the above lemma, note that for full data $s=0$, and $\ratio{s}$ simplifies to $\max_{k,l: k \neq l} ||{U^{(k)}}^{T}U^{(l)}||_{2}$, which is
less than one if and only if there is no intersection between subspaces. In this case, Lemma \ref{lem:noIntersection} states that $\phi_{q}$ is positive (\ie NFC is achievable) as long as there is no intersection between any pair of subspaces.
We next turn to the case where the subspaces are allowed to intersect and points may be corrupted by additive noise.
Lemma~\ref{lem:intersection} \david{bounds $\phi_q$ from below} in such a setting; it requires the subspaces to be sufficiently far apart with respect to their affinity, which is defined as \cite{heckel2015robust, zhang2016global} 
\begin{equation}
    \operatorname{aff}(\sS_{k},\sS_{l}) = \frac{1}{\sqrt{d_{k} \wedge d_{l}}} \norm{U_{k}^{T}U_{l}}_{F},
    \label{eq:affinity}
\end{equation}
where $U_{k}$ and $U_{l}$ form orthonormal bases for the $d_{k}$- and $d_{l}$-dimensional subspaces $\sS_{k}$ and $\sS_{l}$. Note that $\operatorname{aff}(\sS_{k},\sS_{l})$ is a measure of how close two subspaces are in terms of their principal angles and takes the value 1 if two subspaces are equivalent and 0 if they are orthogonal.
\end{sloppypar}

\begin{lemma}[Angular separation for noisy data]
    \label{lem:intersection}
    Let the points in $\sX_{k}$ be the set of $N_{k}$ points $x_{i}^{(k)} = y_{i}^{(k)} + e_{i}^{(k)}$, where \rev{each $y_{i}^{(k)}$ is independently drawn uniform on $\set{y \in \sS_{k}: \norm{y} = 1}$,} and the $e_{i}^{(k)}$ are i.i.d. $\sN(0,\frac{\sigma^{2}}{D}I_{D})$. Let $\sX = \sX_{1} \cup \cdots \cup \sX_{K}$ and $q < N_{min}/6$. Suppose that
    \begin{equation} \label{eq:affCondIntersection}
        \max_{k,l: k \neq l} \operatorname{aff}(\sS_{k},\sS_{l}) + \frac{\sigma(1+\sigma)}{\sqrt{\log N}} \frac{\sqrt{d_{max}}}{\sqrt{D}} \leq \frac{1}{15 \log N},
    \end{equation}
    \rev{and} $D > 6 \log N$. Then \rev{the $q$-angular separation of these noisy points is bounded as} $\phi_{q} > C_2$ with probability at least $1 - \frac{10}{N} - \sum_{k=1}^KN_ke^{-c_2(N_{k}-1)}$, where $c_2 > 0$ is a numerical constant, and $C_2 > 0$ depends only on $\sigma$, $D$, $d_{max}$, $N$, $\max_{k,l: k \neq l} \operatorname{aff}(\sS_{k},\sS_{l})$, and \rev{the function $f$ that the $q$-angular separation is with respect to}. Both $c_2$ and $C_2$ are defined in the proof\rev{,
    and the probability  here is with respect to the randomness from both
    the underlying data $\{y_i^{(k)}\}$ and the noise $\{e_i^{(k)}\}$}.
\end{lemma}
  
Lemmas \ref{lem:noIntersection} and \ref{lem:intersection} state that under certain conditions on the arrangement of subspaces and points, the separation $\phi_{q}$ defined in \eqref{eq:taubound_nfc} is positive with high probability \david{and with given lower bounds}. In the following section, we \david{show that taking sufficiently many base clusterings $B$ in EKSS-0 guarantees the affinity matrix is sufficiently \affclose/ with high probability.}

\subsection{EKSS-0 Recovery Guarantees}
\label{sec:ekssTheory}

\rev{
In this section, we show that the co-association/affinity matrix formed by EKSS with $T = 0$ is angle preserving,
leading to a series of recovery guarantees for the problem of subspace clustering. We refer to the parameter choice of $T = 0$ as \emph{EKSS-0} and include explicit pseudocode for this specialization 
in Appendix~\ref{app:pseudocode}.
}
We say that two points are \emph{co-clustered} if they are assigned to the same candidate subspace in line 5 of Algorithm \ref{alg:EKSS} (note that lines 6-9 are not computed for EKSS-0). The key to our guarantees lies in the fact that for points lying on the unit sphere, the
probability of co-clustering is a monotonically increasing function of the absolute value of their inner product, as shown in Lemma~\ref{lem:coassocProb} below. For EKSS-0, the entries of the \david{affinity} matrix $A$ are empirical estimates of these probabilities, and hence the deviation level $\tau$ is appropriately bounded with high probability by taking sufficiently many base clusterings $B$. These results allow us to apply Theorems~\ref{thm:nfc} and \ref{thm:cdel} from the previous section. We remind the reader that the parameters $\bar{K}$ and $\bar{d}$ are the number and dimension of the \emph{candidate} subspaces in EKSS, and need not be related to the data being clustered.


\begin{theorem}[EKSS-0 is \affclose/]
    \label{thm:ekss0}
    Let $A \in \bR^{N \times N}$ be the affinity matrix formed by EKSS-0 (line 12, Alg. \ref{alg:EKSS}) with parameters $\bar{K}, \bar{d}$ and $B$. Let $\tau > 0$. Then with probability at least $1-N(N-1)e^{-c_3\tau^{2}B}$, 
    the matrix $A$ is $\tau$-\affclose/, where the increasing function $f_{\bar{K},\bar{d}}$ is defined in the proof of Lemma~\ref{lem:coassocProb},
    $c_3 = 2 \sqrt{\log 2}$,
    and the probability is taken with respect to
    the random subspaces drawn in EKSS-0 (line 4, Alg. \ref{alg:EKSS}).
\end{theorem}

In the context of the previous section, Theorem~\ref{thm:ekss0} states that the affinity matrix formed by EKSS-0 is $\tau$-\affclose/ and hence satisfies the main condition required for Theorems~\ref{thm:nfc} and \ref{thm:cdel}.
We refer to the transformation function as $f_{\bar{K},\bar{d}}$ to denote the dependence on the EKSS-0 parameters, noting that $f_{\bar{K},\bar{d}}$ is increasing for \emph{any} natural numbers $\bar{K}$ and $\bar{d}$.
A consequence of Theorem~\ref{thm:ekss0} is that by increasing the number of base clusterings $B$, we can reduce the deviation level $\tau$ to be arbitrarily small while maintaining a fixed probability that the model holds.
This fact allows us to apply the results of the previous section to provide recovery guarantees for EKSS-0.
The major nontrivial aspect of proving Theorem~\ref{thm:ekss0} lies in establishing the following lemma.

\begin{lemma}
    \label{lem:coassocProb}
    The $(i,j)$th entry of the affinity matrix $A$
    formed by EKSS-0 (line 12, Alg. \ref{alg:EKSS})
    has expected value
    \begin{equation}
    \bE A_{i,j} = f_{\bar{K},\bar{d}}(\abs{\ip{x_i}{x_j}})
    \end{equation}
    where
    $f_{\bar{K},\bar{d}}:\bR_+ \to \bR_+$ is a strictly increasing function
    (defined in the proof),
    and the expectation is taken with respect to
    the random subspaces drawn in EKSS-0 (line 4, Alg. \ref{alg:EKSS}).
    The subscripts $\bar{K}$ and $\bar{d}$ indicate
    the dependence of $f_{\bar{K},\bar{d}}$ on those EKSS-0 parameters.
\end{lemma}
\begin{proof}
    We provide a sketch of the proof here; the full proof can be found in Appendix~\ref{app:proofs}. 
   \rev{The proof of this Lemma relies on a geometric understanding of the co-clustering of two points, reducing it to a two-dimensional geometric condition. At this stage, we use a symmetrization trick reminiscent of that used in Vapnik and Chervonenkis's proof of generalization for VC classes. This allows us to derive an easy formula for a conditional probability of co-clustering given the distributional assumptions.} 
    
    For notational compactness, we instead prove that the probability \laura{of two points being co-clustered is a \emph{decreasing} function of the angle $\theta$ between them.} 
Denote this probability by $p_{\bar K,\bar d}(\theta)$. Let $U_{1}, U_{2}, \dots, U_{\bar K} \in \bR^{D \times \bar d}$ be the $K$ candidate bases. Let $\tilde{p}(\theta)$ be the probability that any two points with corresponding angle
    $\theta$ are assigned to the candidate $U_{1}$. Then \david{by symmetry} we have $p_{\bar K, \bar d}(\theta) = K \tilde{p}(\theta)$, and it suffices to prove that $\tilde{p}$ is strictly decreasing. Without loss of generality, let $x_{i} = e_{1}$ and $x_{j} = \cos(\theta) e_{1} + \sin(\theta) e_{2}$, where $e_{m} \in \bR^{D}$ is the $m$th standard basis vector. We then have
    that
    \begin{equation*}
        \tilde{p}(\theta) = \bP \set{Qx_{i}, Qx_{j} \text{ both assigned to } U_{1}},
    \end{equation*}
    where $Q$ is an arbitrary orthogonal transformation of $\bR^{D}$. Let $E$ denote the event of interest and $L$ denote the span of $e_{1}$ and $e_{2}$. The event $E$ can then be written as
    \begin{eqnarray}
        \label{eq:quadratic}
        z^{T}QP_{L}(P_{1} - P_{k})P_{L}Qz > 0, \quad \text{for} \quad 1 < k \leq K \quad \text{and} \quad z = x_{i}, x_{j},
    \end{eqnarray}
    where $P_{L}$ denotes the orthogonal projection onto the subspace $L$ and $P_{k}$ denotes the orthogonal projection onto the subspace spanned by $U_{k}$. By restricting to $L$, \eqref{eq:quadratic} can be reduced to a two-dimensional quadratic form, and we can compute in closed form $\bP \set{E \given U_{1}, \dots, U_{\bar K}}$. Differentiating shows that this term is decreasing and hence (by the law of total probability) so is $\tilde{p}(\theta)$.
\end{proof}



It is interesting to note that the result of Lemma \ref{lem:coassocProb} does not depend on the underlying data distribution, \ie the number or arrangement of subspaces, but instead says that clustering with EKSS-0 is (in expectation) a function of the absolute inner product between points, regardless of the parameters. Thus, the results of this section all hold even with the simple parameter choice of $\bar{K} = 2$ and $\bar{d} = 1$ in Algorithm \ref{alg:EKSS}. 
Our empirical results suggest that \laura{decreasing $\bar{K}$ and} increasing $\bar{d}$ increases the probability of co-clustering. 
However, when running several iterations of KSS (EKSS with
$T > 0$), we find that it is advantageous to choose $\bar{K}$ and $\bar{d}$ to match the true parameters of the data as closely as possible, allowing KSS to more accurately model the underlying subspaces. 

Combined with the results of Section \ref{sec:generalTheory}, Theorem~\ref{thm:ekss0} enables us to quickly obtain recovery guarantees for EKSS-0, which we now present.
We first consider the case where the data are noiseless, \ie lie perfectly on a union of $K$ subspaces. Theorems~\ref{thm:noIntersectionFull} and \ref{thm:intersectionNoiseless} provide sufficient conditions on the arrangement of subspaces such that EKSS-0 achieves \emph{correct clustering} with high probability.

\begin{theorem}[EKSS-0 provides correct clustering for disjoint subspaces]
    \label{thm:noIntersectionFull}
    Let $\sS_{k}$, $k=1,\dots,K$ be subspaces of dimension $d_1,\dots,d_K$ in $\bR^D$. Let the $N_{k}$ points in $\sX_{k}$ be drawn as $x_{j}^{(k)} = U^{(k)}a_{j}^{(k)}$, where $a_{j}^{(k)}$ are i.i.d. uniform on $\bS^{d_{k}-1}$ and $U^{(k)} \in \bR^{D \times d_{k}}$ has orthonormal columns that form a basis for $\sS_k$.
    Let $\rho \in [0,1)$ be arbitrary and suppose that $N_{min} > N_{0}$, where $N_{0}$ is a constant that depends only on $d_{max}$ and $\rho$. Suppose that $q \in [c_{4} \log N_{max},N_{min}^{\rho}]$ and
    \begin{equation}
        \ratio{0} = \max_{k,l: k \neq l} \norm{{U^{(k)}}^{T}U^{(l)}}_{2} < 1,
        \label{eq:affCond}
    \end{equation}
    where $c_{4} = 12(24\pi)^{d_{max}-1}$.
    Then $\bar{A}$ obtained by EKSS-0 results in correct clustering of the data with probability at least $1 - \sum_{k = 1}^{K} \left( N_{k} e^{-c_1(N_{k}-1)} + 2N_{k}^{-2} \right) - N(N-1)e^{-c_3B\min\set{C_1,C_{3}}^2}$, where $c_1,c_3 > 0$ are numerical constants, $C_1 > 0$ depends on $\ratio{0}$ and the function $f_{\bar{K},\bar{d}}$ defined in Theorem~\ref{thm:ekss0}, and $C_{3} > 0$ depends on $d_{max}$, $N_{min}$, and $f_{\bar{K},\bar{d}}$.
\end{theorem}


\begin{theorem}[EKSS-0 provides correct clustering for subspaces with bounded affinity]
    \label{thm:intersectionNoiseless}
    Let $\sS_{k}$, $k=1,\dots,K$ be subspaces of dimension $d_1,\dots,d_K$ in $\bR^D$.
    Let the points in $\sX_{k}$ be a set of $N_{k}$ points drawn uniformly from the unit sphere in subspace $k$, \ie from the set $\set{x \in \sS_{k}: \norm{x} = 1}$. Let $\sX = \sX_{1} \cup \cdots \cup \sX_{K}$ and $N=\sum_k N_k$. Let $q \in \left[ c_{4} \log N_{max}, N_{min}/6 \right)$, where $c_{4} =
    12(24\pi)^{d_{max}-1}$. If
    \begin{equation*}
        \max_{k,l: k \neq l} \operatorname{aff}(\sS_{k},\sS_{l}) \leq \frac{1}{15 \log N},
    \end{equation*}
    then $\bar{A}$ obtained by EKSS-0 results in correct clustering of the data with probability at least $1 - \frac{10}{N} - \sum_{k = 1}^{K} \left( N_{k} e^{-c_2(N_{k}-1)} - 2N_{k}^{-2} \right) - N(N-1)e^{-c_3B\min\set{C_2,C_{3}}^{2}}$, where $c_2,c_3 > 0$ are numerical constants, $C_2 > 0$ depends only on $\max_{k,l: k \neq l} \operatorname{aff}(\sS_{k},\sS_{l})$, $D$, $d_{max}$, $N$, and the function $f_{\bar{K},\bar{d}}$ defined in Theorem~\ref{thm:ekss0}, and $C_{3} > 0$
    depends on $d_{max}$, $N_{min}$, and $f_{\bar{K},\bar{d}}$.
\end{theorem}


We next consider two forms of data corruption. Theorem~\ref{thm:noisy} shows that the affinity matrix built by EKSS-0 has NFC in the presence of data corrupted by additive Gaussian noise. Theorem~\ref{thm:noIntersectionMissing} shows that EKSS-0 maintains NFC even in the presence of a limited number of missing (unobserved) entries.

\begin{theorem}[EKSS-0 has NFC with noisy data]
    \label{thm:noisy}
    Let $\sS_{k}$, $k=1,\dots,K$ be subspaces of dimension $d_1,\dots,d_K$ in $\bR^D$.
    Let the points in $\sX_{k}$ be the set of $N_{k}$ points $x_{i}^{(k)} = y_{i}^{(k)} + e_{i}^{(k)}$, where the $y_{i}^{(k)}$ are drawn i.i.d. from the set $\set{y \in \sS_{k}: \norm{y} = 1}$, independently across $k$, and the $e_{i}^{(k)}$ are i.i.d. $\sN(0,\frac{\sigma^{2}}{D}I_{D})$. Let $\sX = \sX_{1} \cup \cdots \cup \sX_{K}$ and $q < N_{min}/6$. If
    \begin{equation*}
        \max_{k,l: k \neq l} \operatorname{aff}(\sS_{k},\sS_{l}) + \frac{\sigma(1+\sigma)}{\sqrt{\log N}} \frac{\sqrt{d_{max}}}{\sqrt{D}} \leq \frac{1}{15 \log N},
    \end{equation*}
    with $D > 6 \log N$, then $\bar{A}$ obtained from running EKSS-0 has no false connections with probability at least $1 -  \frac{10}{N} - \sum_{k = 1}^{K} N_{k} e^{-c_2(N_{k}-1)} - N(N-1)e^{-c_3C_2^{2}B}$, where $c_2, c_3 > 0$ are numerical constants, and $C_2 > 0$ depends only on $\max_{k \neq l} \text{aff}\left( \sS_{k}, \sS_{l} \right)$, $\sigma$, $D$, $d_{max}$, $N$ and the function $f_{\bar{K},\bar{d}}$ defined in Theorem~\ref{thm:ekss0}.
\end{theorem}


\begin{theorem}[EKSS-0 has NFC with missing data]
    \label{thm:noIntersectionMissing}
    Let the $n$ points in $\sX_{k}$ be drawn as $x_{j}^{(k)} = U^{(k)}a_{j}^{(k)}$, where $a_{j}^{(k)}$ are i.i.d. uniform on $\bS^{d-1}$
    and the entries of $U^{(k)} \in \bR^{D \times d}$ are i.i.d. $\sN(0,\frac{1}{D})$.
    Let $\rho \in [0,1)$ be arbitrary and suppose that $n > N_{0}$, where $N_{0}$ is a constant that depends only on $d$ and $\rho$. Suppose that $q < n^{\rho}$, and assume that in each $x_{j} \in \sX$ up to $s$ arbitrary entries are unobserved, \ie set to 0. Let $\sX = \sX_{1} \cup \cdots \cup \sX_{K}$. If
    \begin{equation}
        D \laura{-3c_{5}d - c_{5} \log K} \geq s\left( c_{5}\log\left( \frac{De}{2s} \right) + c_{6} \right),
        \label{eq:missCond}
    \end{equation}
    then $\bar{A}$ obtained by EKSS-0 has no false connections with probability at least $1 - Ne^{-c_{1}(n-1)} - N(N-1)e^{-c_{3}C_{1}^{2}B} - 4e^{-c_{7}D}$, where 
    $c_{1}, c_{3}, c_{5}, c_{6}, c_{7} > 0$, are numerical constants and $C_{1} > 0$ depends only on the ratio $\ratio{s}$ defined in \eqref{eq:affCondMissing} and the function $f_{\bar{K},\bar{d}}$ defined in Theorem~\ref{thm:ekss0}.
\end{theorem}


\subsection{Discussion of Results}
\label{sec:resultsdiscuss}

The data model considered in Theorems \ref{thm:noIntersectionFull}-\ref{thm:noIntersectionMissing} is known as the ``semi-random'' model \cite{soltanolkotabi2012geometric}, due to the fixed arrangement of subspaces with randomly-drawn points, and has been analyzed widely throughout the subspace clustering literature \cite{soltanolkotabi2012geometric,soltanolkotabi2014robust,heckel2014subspace,heckel2015robust,wang2016graph}.
Our guarantees under this model are identical (up to constants and log factors) to those for TSC and SSC (see \cite[Section VII]{heckel2015robust} for further discussion of their guarantees). The key difference between our results and those of TSC is that we pay at most a $N(N-1) e^{-c_{3}\min{\set{C_{1},C_{2},C_{3}}}^{2}B}$ penalty in recovery probability due to the approximate observations of the transformed inner products. Although our experiments indicate that EKSS-0 appears to have no benefits over TSC, we do find that by running
a small number of KSS iterations, significant performance improvements are achieved. \laura{While the above analysis holds only for the case of $T = 0$, letting $T > 0$ is guaranteed to not increase the KSS cost function \cite{bradley2000kplane}. In our experiments, we found that setting $T > 0$ uniformly improved clustering performance, and our empirical results indicate that EKSS is in fact more robust (than EKSS-0 and TSC) to subspaces with small principal angles. }

While the explicit choice of $B$ is tied to the unknown function $f_{\bar{K},\bar{d}}$, our results provide intuition for setting this value; namely, the
closer the underlying subspaces (in terms of principal angles), the more \david{base clusterings} required. The inverse dependence on $\log N$ in Theorems \ref{thm:intersectionNoiseless} and \ref{thm:noisy} indicates a tension as the problem size grows. On one hand, \david{points from the same subspace are more likely to be close when $N$ is large, improving the angular separation}. On the other \david{hand, points are also more likely to fall near the intersection of subspaces, potentially degrading the angular separation.}
In all experimental results, we see that both EKSS and TSC perform better with larger $N$. Finally, we note that the
leading $O\left( N^{2} \right)$ coefficient in the above probabilities results from applying a union bound \david{and is likely conservative.}

\rev{
\subsection{Additional Recovery Guarantees}
\label{sec:addguarantees}

As mentioned at the start of Section \ref{sec:theory}, our recovery guarantees have application beyond the analysis of EKSS-0. In this section, we show that our framework for analyzing angle-preserving affinities strengthens the results of \cite{heckel2017dimensionality} to show that TSC yields correct clustering (as opposed to NFC only) after linear dimensionality reduction. We finally show empirically that this holds also for EKSS-0, when clustering is applied to data that have been transformed by both linear and nonlinear dimensionality reduction. 


We first state our result for TSC in the widely-studied case of linear dimensionality reduction. Unlike the results of \cite{heckel2017dimensionality,wang2018theoretical}, our framework allows us to prove that TSC achieves correct clustering.
We consider a linear dimensionality reduction to $p$ dimensions. Let $\Phi \in \bR^{p \times D}$ and suppose that for all $x \in \sX$ simultaneously, we have
\begin{equation}
    \label{eq:approxIsom}
    (1-\tau) \norm{x}_{2}^{2} \leq \norm{\Phi x}_{2}^{2} \leq (1+\tau) \norm{x}_{2}^{2}
\end{equation}
with probability at least $1 - 2e^{-c_{3}\tau^{2}p}$.
This holds, e.g., for random projections
as long as $p > (C/\tau^{2}) \log N$ \cite{vershynin2018high}.

\begin{theorem}[TSC provides correct clustering on dimensionality-reduced data]
    \label{thm:dimReducedTSC}
    Consider the setting of Thm.~\ref{thm:intersectionNoiseless} and assume that
    \begin{equation*}
        \max_{k,l: k \neq l} \operatorname{aff}(\sS_{k},\sS_{l}) \leq \frac{1}{15 \log N}.
    \end{equation*}
    Assume dimensionality reduction satisfying \eqref{eq:approxIsom} is applied to the data $\sX$ before clustering. Then $\bar{A}$ obtained by TSC results in correct clustering of the data with probability at least
    \begin{equation*}
        1 - \frac{10}{N} - \sum_{k = 1}^{K} \left( N_{k} e^{-c_2(N_{k}-1)} - 2N_{k}^{-2} \right) - 2e^{-\tilde{c}\min\set{C_{2},C_{3}}^{2}p},
    \end{equation*}
where $c_2, c_{3} > 0$ are numerical constants, $\tilde{c}$ is the constant from \eqref{eq:approxIsom}, $C_2 > 0$ depends only on \\ $\max_{k,l: k \neq l} \operatorname{aff}(\sS_{k},\sS_{l})$, $D$, $d_{max}$, and $N$, and $C_{3} > 0$
    depends on $d_{max}$ and $N_{min}$. 
\end{theorem}

To compare our results with those of \cite{heckel2017dimensionality}, we consider the requirement for NFC only.
To achieve NFC with high probability,
$\tau < C_2$ alone is sufficient
(by Theorem 4.4 and Lemma 4.2).
Specializing the proof of Lemma 4.2
to this dimensionality reduced setting yields
the sufficient condition
\begin{equation*}
    \tau < \frac{1}{3\sqrt{d}} - \frac{2 \sqrt{6}}{15 \sqrt{d}} = \frac{\bar{c}}{\sqrt{d}},
\end{equation*}
where $\bar{c} = \frac{5 - 2\sqrt{6}}{15}$.
This condition can be met by random projections
with probability at least $1 - 2e^{-c_{3}\tau^{2}p}$
as long as
\begin{equation*}
    p > \frac{C}{\tau^{2}} \log N > \frac{C}{\bar{c}^{2}} d \log N.
\end{equation*}

The NFC condition for TSC is also analyzed in \cite[Theorem 3.1]{heckel2017dimensionality}.
Their analysis does not incur the probability penalty incurred in our analysis, but they instead require a stricter condition on the affinity:
\begin{equation*}
    \max_{k \neq l} \text{aff}(\sS_{k},\sS_{l}) \leq \frac{1}{15 \log N} - \frac{\sqrt{11}}{\sqrt{3c_{3}}} \frac{\sqrt{d}}{\sqrt{p}},
\end{equation*}
where $c_{3}$ is the same as above.
Since affinity is nonnegative,
this condition is only feasible when
\begin{equation*}
    \frac{1}{15 \log N} - \frac{\sqrt{11}}{\sqrt{3c_{3}}} \frac{\sqrt{d}}{\sqrt{p}} \geq 0,
\end{equation*}
or equivalently,
\begin{equation*}
    p \geq \frac{825}{c_{3}} d \log^2 N.
\end{equation*}
Hence, both analyses result in a similar lower bound on the projected dimension, with the tradeoff being that \cite[Theorem 3.1]{heckel2017dimensionality} requires a stricter assumption on the affinity, whereas our analysis allows for a slightly higher probability of failure. However, our analysis also yields the novel result of \textit{correct clustering} for TSC as long as $\tau$ is sufficiently small.

The above result covers the case of linear dimensionality reduction in the simple case where $A_{i,j} = \abs{\ip{x_{i}}{x_{j}}}$. Extending to nonlinear dimensionality reduction and/or the interaction of dimensionality reduction operations with more complex notions of similarity in clustering, such as that obtained by EKSS/EKSS-0, is challenging due to (1) the lack of theoretical guarantees on popular dimensionality reduction techniques such as 
UMAP \cite{mcinnes2018umap} and (2) the behavior of the composition of dimensionality reduction with the monotonic functions from EKSS-0 whose precise form is unknown. 
That said, a common feature of dimensionality reduction techniques is preservation of local distances and local angles, which is a key principle in our theory and a key attribute of EKSS-0.

Therefore we here provide empirical evidence that the affinity produced by EKSS-0 is $\tau$-angle preserving when applied after both linear and nonlinear dimensionality reduction; See Fig.~\ref{fig:dimReduction}. The figure shows the empirical probability of co-clustering as a function of angle between points after clustering via EKSS-0 on full data, data reduced via a random Gaussian projection, and data reduced via UMAP. For both the full data and linear dimensionality reduction cases, the resulting probability monotonically
decreases with angle, i.e., is monotonically increasing with absolute inner product as desired. For UMAP, the co-clustering probability is monotone until the angle between vectors is approximately 0.2 radians. While the function is not entirely monotone, if enough data is available, it may be likely that the $q$ nearest neighbors of each point lie within the monotone region, making the NFC and connectedness results in Theorems \ref{thm:nfc} and \ref{thm:cdel} applicable. Generalizing our theory to this setting, where the monotonicity of the function degrades as angles become orthogonal, would be an interesting future direction. }

\begin{figure}[t]
 \rev{
    \centering
\includegraphics[width=0.4\textwidth]{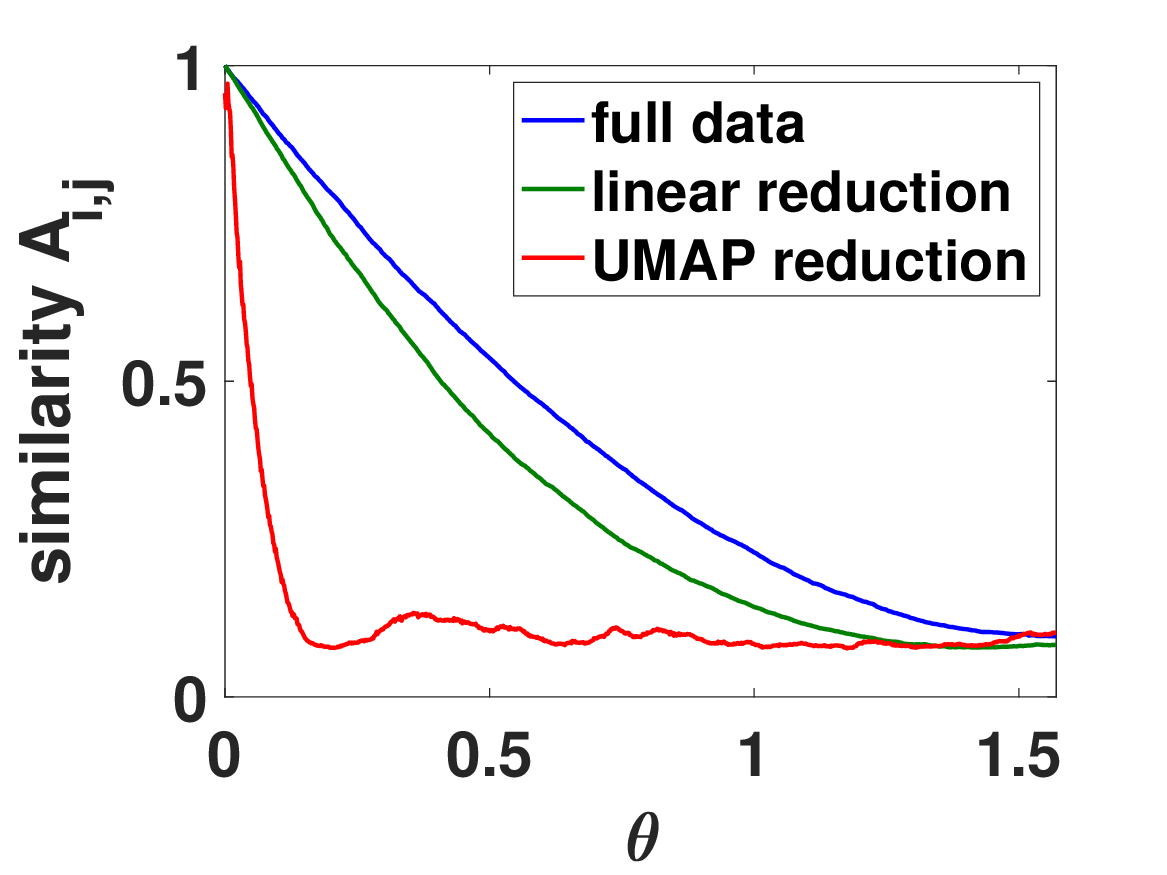}
    \caption{Empirical estimate of similarity $A_{i,j}$ as a function of angle between vectors when clustering via EKSS-0. Points are drawn from $\bR^{100}$ and reduced to $\bR^{20}$ using linear dimensionality reduction via Gaussian random projection and nonlinear dimensionality reduction via UMAP with cosine similarity. EKSS-0 uses $\bar{K} = 10$ candidate subspaces of dimension $\bar{d} = 2$. Clustering probability (similarity) is monotonically decreasing in angle after applying linear dimensionality reduction, and for UMAP up
    to an angle of 0.2 radians.}
    \label{fig:dimReduction}
    }
\end{figure}

\section{Experimental Results}
\label{sec:experiments}

\sloppypar In this section, we demonstrate the performance in terms of clustering error (defined in Appendix~\ref{app:clustErr}) of EKSS on both synthetic and real datasets. 
We first show the performance of our algorithm as a function of the relevant problem parameters and verify that EKSS-0 exhibits the same empirical performance as TSC, as expected based on our theoretical guarantees. We also show that EKSS can recover subspaces that either have large intersection or are extremely close.
We then demonstrate on benchmark datasets that EKSS not only improves over previous geometric methods, but that it achieves state-of-the-art results competitive with those obtained by self-expressive methods.

\subsection{Synthetic Data}

For all experiments in this section, we take $q = \max(3,\left\lceil N_{k}/20 \right\rceil)$ for EKSS-0 and TSC and $q = \max(3,\left\lceil N_{k}/6 \right\rceil)$ for EKSS, where $\left\lceil c \right\rceil$ denotes the largest integer greater than or equal to $c$. We set $B = 10,000$ for EKSS-0 and EKSS. When the angles between subspaces are not explicitly specified, it is assumed that the subspaces are drawn uniformly at random from the set of all $d$-dimensional subspaces of $\bR^{D}$.
For all experiments, we draw points uniformly at random from the unit sphere in the corresponding subspace and show the mean error over 100 random problem instances. We use the code provided by the authors for TSC and SSC. We employ the ADMM implementation of SSC and choose the parameters that result in the best performance in each scenario. 

\begin{figure*}[t]
    \centering
\includegraphics[width=\textwidth]{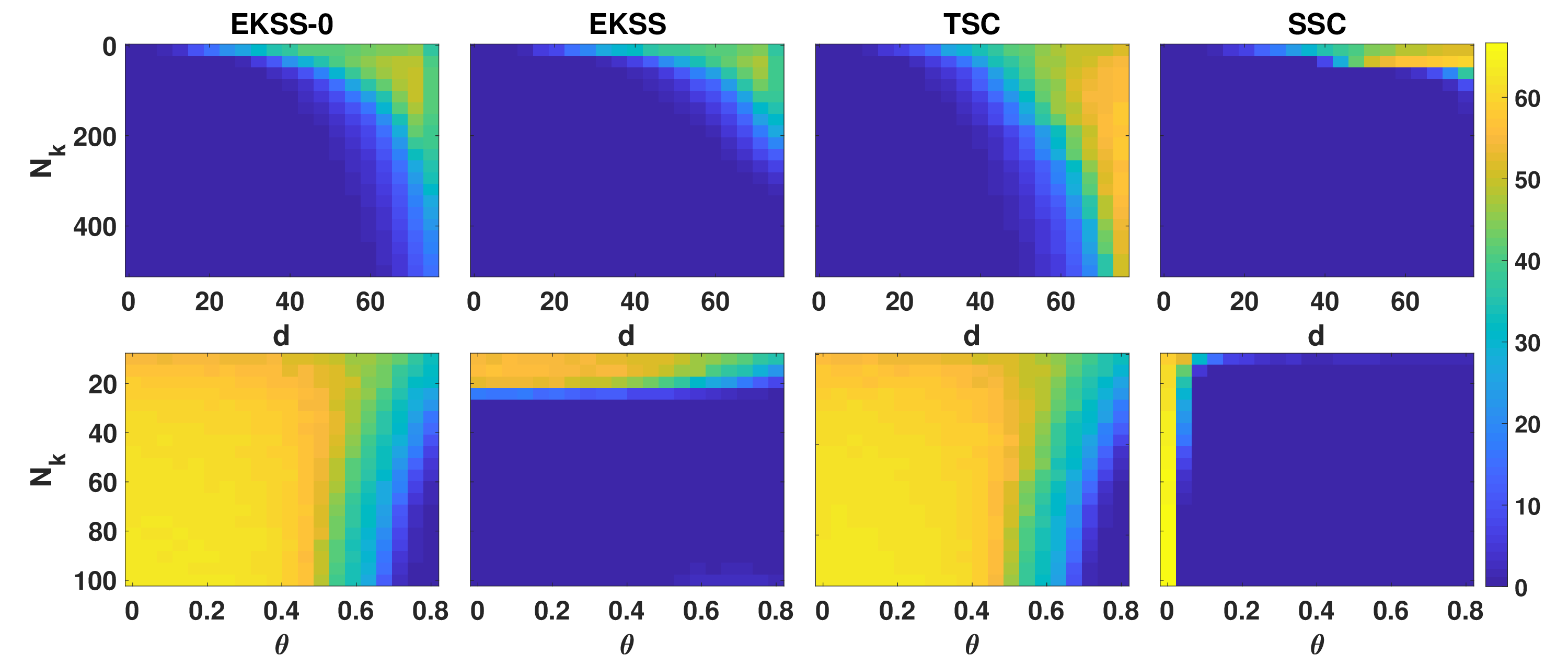}
    \caption{Clustering error (\%) for proposed and state-of-the-art subspace clustering algorithms as a function of problem parameters $N_k$, number of points per subspace, and true subspace dimension $d$ or angle between subspaces $\theta$. Fixed problem parameters are $D = 100$, $K = 3$.}
    \label{fig:heatmaps}
\end{figure*}

We explore the influence of some relevant problem parameters on the EKSS algorithm in Fig. \ref{fig:heatmaps}.
We take the ambient dimension to be $D = 100$, the number of subspaces to be $K = 3$, and generate noiseless data.
We first consider the dependence on subspace dimension and the number of points per subspace. The top row of Fig. \ref{fig:heatmaps} shows the misclassification rate as the number of points per subspace ranges from $10-500$ and the
subspace dimension ranges from $1-75$. When $2d>D$ (\ie $d \geq 51$), pairs of subspaces necessarily have intersection, and the intersection dimension grows with $d$. First, the figures demonstrate that EKSS-0 achieves roughly the same performance as TSC, resulting in correct clustering even in the case of subspaces with large intersection. Second, we see that EKSS can correctly cluster for subspace dimensions larger than that of TSC as long as there are sufficiently many points per subspace.
For large subspace dimensions with a moderate number of points per subspace, SSC achieves the best performance.

We next explore the clustering performance as a function of the distance between subspaces, as shown in the second row of Fig. \ref{fig:heatmaps}. We set the subspace dimension to
$d = 10$ and generate $K = 3$ subspaces such that the principal angles between subspaces $\sS_{1}$ and $\sS_{2}$, as well as those between $\sS_{1}$ and $\sS_{3}$ are $\theta$, for 20 values in the range $\left[ 0.001, 0.8 \right]$. Most strikingly, EKSS is able to resolve subspaces with even the smallest separation. This stands in contrast to \david{TSC; it} fails in this regime because when the subspaces are extremely close, the inner products between points on different subspaces can be nearly as large as those within the same
subspace. 
Similarly, in the case of SSC, points on different subspaces can be used to regress any given point with little added cost, and so it fails at very small subspace angles. 
However, as long as there is still some separation between subspaces, EKSS is able to correctly cluster all points. 
The theory presented here does not capture this phenomenon, and recovery guarantees that take into account multiple iterations of KSS are an important topic for future work.


\begin{figure}[t]
    \centering
\includegraphics[width=0.4\textwidth]{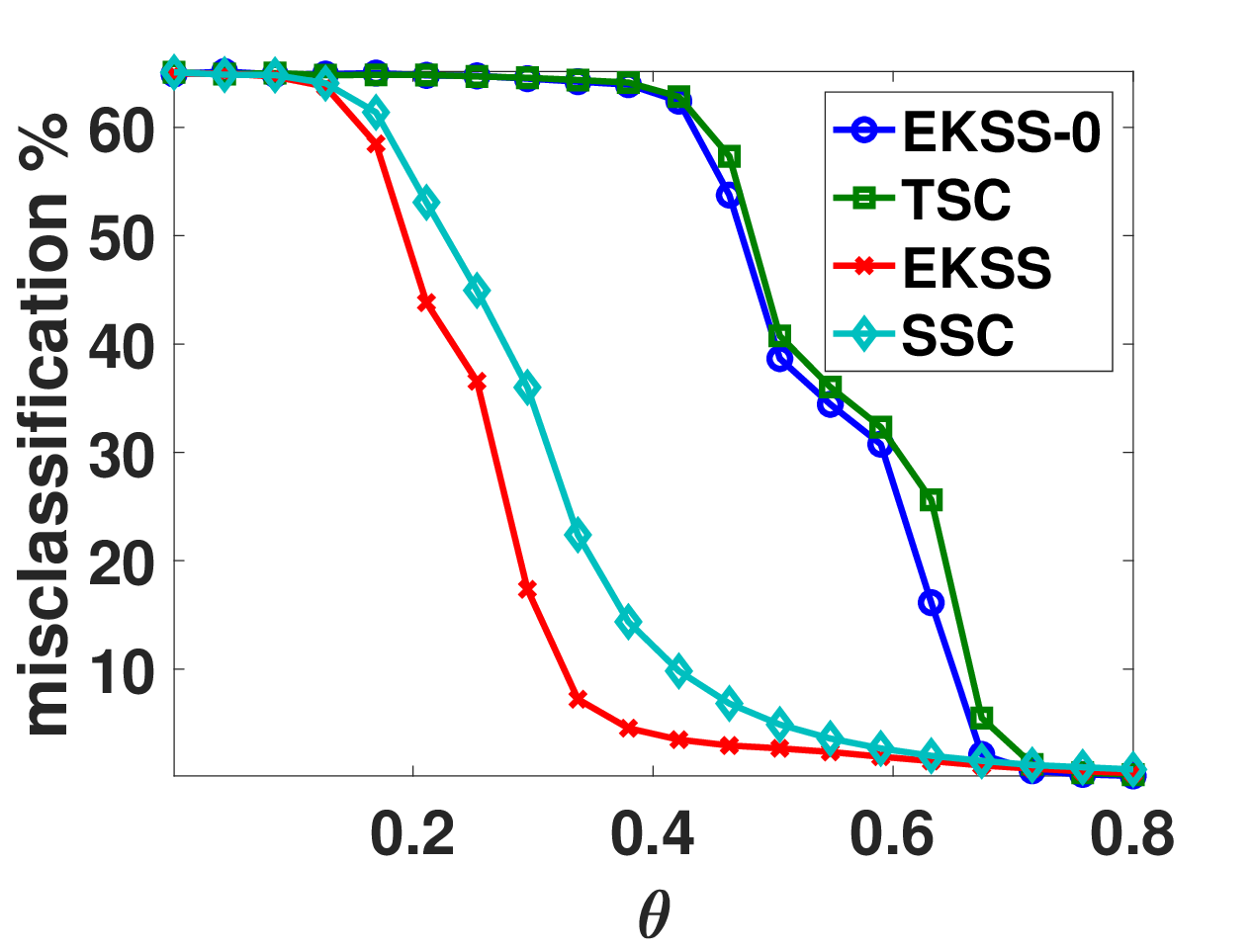}
    \caption{Clustering error (\%) as a function of subspace angles with noisy data. Problem parameters are $D = 100$, $d = 10$, $K = 3$, $N_{k} = 500$, $\sigma^{2} = 0.05$.}
    \label{fig:comp}
\end{figure}

As a final comparison, we show the clustering performance with noisy data. Fig. \ref{fig:comp} shows the clustering error as a function of the angle between subspaces for the case of $K = 3$ subspaces of dimension $d = 10$, with $N_{k} = 500$ points corrupted by zero-mean Gaussian noise with covariance $0.05 I_{D}$. We again consider 20 values of the angle $\theta$ between 0.001 and 0.08. EKSS-0 and TSC
obtain similar performance, and more importantly EKSS is more robust
to small subspace angles than SSC, even in the case of noisy data.

\subsection{Benchmark Data}

In this section, we show that EKSS achieves competitive subspace clustering performance on a variety of datasets commonly used as benchmarks in the subspace clustering literature. We consider the Hopkins-155 dataset \cite{tron2007benchmark}, the cropped Extended Yale Face Database B \cite{georghiades2001from,lee2005acquiring}, COIL-20 \cite{nene1996coil20} and COIL-100 \cite{nene1996coil100} object databases, the USPS dataset provided by
\cite{cai11graph}, and 10,000 digits of the MNIST handwritten digit database \cite{lecun2016mnist}, where we obtain features using a scattering network \cite{bruna2013invariant} as in \cite{you2016oracle}. Descriptions of these datasets and the relevant problem parameters are included in Appendix~\ref{app:expDetails}. We compare the
performance of EKSS to several benchmark algorithms: KSS \cite{bradley2000kplane}, CoP-KSS \cite{gitlin2018improving}, Median K-Flats (MKF) \cite{zhang2009median}, TSC \cite{heckel2015robust}, the ADMM implementation of SSC \cite{elhamifar2013sparse}, SSC-OMP \cite{you2016scalable}, and Elastic Net Subspace Clustering (EnSC) \cite{you2016oracle}. For all algorithms, we selected the parameters that yielded the lowest clustering error, performing extensive model
selection where possible. We point out that this method of parameter selection requires knowledge of the ground truth labels, which are typically unavailable in practice. 
For the larger USPS and MNIST datasets, we obtained a small benefit by replacing PCA (line 7, Alg.~\ref{alg:EKSS}) with the more robust Coherence Pursuit, i.e., we use CoP-KSS as a base clustering algorithm instead of KSS.
Further implementation details, including parameter selection and data preprocessing, can be found in Appendix~\ref{app:expDetails}.

\begin{table*}[t]
    \centering
    \begin{tabular}{ | c || c | c | c | c | c | c | }
        \hline 
        Algorithm & Hopkins & Yale B & COIL-20 & COIL-100 & USPS & MNIST-10k \\
        \hline
        EKSS & \textbf{0.26} & 14.31 & 13.47 & \textbf{28.57} & \textbf{15.84} & \textbf{2.39} \\
        \hline
        KSS & 0.35 & 54.28 & 33.12 & 66.04 & 18.31 & 2.60 \\
        \hline
        CoP-KSS & 0.69 & 52.59 & 29.10 & 51.38 & \textbf{7.73} & \textbf{2.57} \\
        \hline
        MKF & \textbf{0.24} & 41.32 & 35.69 & 59.50 & 28.49 & 28.17 \\
        \hline
        TSC & 2.07 & 22.20 & 15.28 & 29.82 & 31.57 & 15.98 \\
        \hline
        SSC-ADMM & 1.07 & \textbf{9.83} & \textbf{13.19} & 44.06 & 56.61 & 19.17 \\
        \hline
        SSC-OMP & 25.25 & \textbf{13.28} & 27.29 & 34.79  & 77.94 & 19.19 \\
        \hline
        EnSC & 9.75 & 18.87 & \textbf{8.26} & \textbf{28.75} & 33.66 & 17.97 \\
        \hline
    \end{tabular}
    \caption{Clustering error (\%) of subspace clustering algorithms for a variety of benchmark datasets. The lowest two clustering errors are given in bold. Note that EKSS is among the best three for all datasets, but no other algorithm is in the top five across the board.}
    \label{tab:real}
\end{table*}

The clustering error for all datasets and algorithms is shown in Table \ref{tab:real}, with the lowest two errors given in bold. First, note that EKSS outperforms its base clustering algorithm (KSS or CoP-KSS) in all cases except the USPS dataset, and sometimes by a very large margin. This result emphasizes the importance of leveraging all clustering information from the $B$ base clusterings, as opposed to simply choosing the best single clustering.
While CoP-KSS achieves lower clustering error than EKSS on the USPS dataset, 
a deeper investigation of the performance of CoP-KSS revealed that only 17 of the 1000 individual clusterings achieved an error lower than the 15.84\% obtained by EKSS.
A more sophisticated weighting scheme than that described in Section~\ref{sec:clusteringAccuracy} could be employed to add more significant weights for the small number of base clusterings corresponding to low error.
Alternative measures of clustering quality based on subspace margin \cite{lipor2017leveraging} or novel internal clustering validation metrics \cite{lipor2018clustering} may provide improved performance.
Next, the results show that EKSS is among the top performers in all datasets considered, achieving nearly perfect clustering of the Hopkins-155
dataset, which is known to be well approximated by the UoS model. Scalable algorithms such as SSC-OMP and EnSC perform poorly on this dataset, likely due to the small number of points. For the larger COIL-100, USPS, and MNIST datasets, EKSS also achieves strong performance, demonstrating its flexibility to perform well in both the small and large sample regimes. 
The self-expressive methods outperform EKSS on the Yale and COIL-20 datasets, likely due to
the fact that they do not explicitly rely on the UoS model in building the affinity matrix. However, EKSS still obtains competitive performance on both datasets, making it a strong choice for a general-purpose algorithm for subspace clustering. 

\section{Conclusion}
\label{sec:conclusion}

In this work, we presented the first known theoretical guarantees for both evidence accumulation clustering and the KSS algorithm.
We showed that with a given choice of parameters, the EKSS algorithm can provably cluster data from a union of subspaces under the same conditions as existing algorithms. The theoretical guarantees presented here match existing guarantees in the literature, and our experiments on synthetic data indicate that the iterative approach of KSS provides a major improvement in robustness to small angles between subspaces. Further, our results generalize those in the existing literature, yielding the potential to inform future algorithm design and analysis. We demonstrated the efficacy of our approach on both synthetic and real data, and showed that our method achieves excellent performance on several real datasets.

A number of important open problems remain. First, extending our analysis to the general case of Alg.~\ref{alg:EKSS} (\ie $T > 0$) is an important next step that is difficult because of the alternating nature of KSS. In selecting tuning parameters, we chose the combination that resulted in the lowest clustering error, which is not known in practice. Methods for unsupervised model selection are an important practical consideration for EKSS and subspace clustering in general. Finally, while we did not have success in implementing ensembles of state-of-the-art algorithms such as SSC,
a deeper study of this topic could yield improved empirical performance.

\section{Acknowledgments}
\label{sec:acknowledgment}

J. Lipor was supported by NSF GRFP award F031543-071159, DARPA grant 16-43-D3M-FP-037, and by the U.S. Army Basic Research Program under PE 61102, Project T25, Task 02 ``Network Science Initiative,'' managed at the U.S. Army ERDC with Portland State University under Cooperative Agreement Number W912HZ-17-2-0005.
D. Hong was supported by NSF GRFP award DGE1256260, NSF BIGDATA grant IIS 1837992, and the Dean's Fund for Postdoctoral Research of the Wharton School;
work was done in part while D. Hong was a graduate student at the University of Michigan.
Y.S. Tan was supported by NSF TRIPODS and the Simons Institute for the Theory of Computing.
L. Balzano was supported by DARPA 16-43-D3M-FP-037, NSF CCF-1845076 and IIS-1838179, AFOSR FA9550-19-1-0026, and ARO W911NF1910027.

\appendix

\section{Proofs of Theoretical Results}
\label{app:proofs}

The results of this section make use of the following notation. We define the absolute inner product between points $x_{i} \in \sS_{l}$ and $x_{j} \in \sS_{k}$ as
\begin{equation*}
    z_{i,j}^{(l,k)} = \abs{\ip{x_{i}^{(l)}}{x_{j}^{(k)}}},
\end{equation*}
where $k$ may be equal to $l$. We denote the $q$th largest absolute inner product between $x_{i}^{(l)}$ and other points in the subspaces $\sS_{l}$ as $z_{(i,q)}^{(l)}$, i.e., we have
\begin{equation*}
    z_{(i,q)}^{(l)} = \abs{\ip{x_{i}^{(l)}}{x_{\neq i}^{(l)}}}_{[q]}
\end{equation*}
in the context of Definition~\ref{def:gap}.

\subsection{Proof of Theorem~\ref{thm:nfc}}
    We first prove the statement for a fixed $x_{i} \in \sS_{l}$. The statement of the theorem can be written as
    \begin{equation}
        \hat{f}_{(i,q)}^{(l)} > \max_{k \neq l, j} \hat{f}_{i,j}^{(l,k)},
        \label{eq:nfc1}
    \end{equation}
    where $\hat{f}_{(i,q)}^{(l)}$ denotes the $q$th largest value in the set $\set{\hat{f}_{i,j}^{(l,l)}}$.
    We first bound $\hat{f}$ in terms of $f$. Let $x_{\iota} \in \sS_{k^{*}}$ be such that $\max_{k \neq l, j} \hat{f}_{i,j}^{(l,k)} = \hat{f}_{i,\iota}^{(l,k^{*})}$ and note that $z_{i,\iota}^{(l,k^{*})} \leq \max_{k \neq l, j} z_{i,j}^{(l,k)}$. Then we have
        \begin{eqnarray*}
            \max_{k \neq l, j} \hat{f}_{i,j}^{(l,k)} = \hat{f}_{i,\iota}^{(l,k^{*})} &\leq& f\left( z_{i,\iota}^{(l,k^{*})} \right) + \tau \\
            &\leq& f\left( \max_{k \neq l, j} z_{i,j}^{(l,k)} \right) + \tau,
        \end{eqnarray*}
        where the second line follows by monotonicity of $f$.
        To lower bound $\hat{f}_{(i,q)}^{(l)}$, let $x_{\kappa}$ be such that $\hat{f}_{(i,q)}^{(l)} = \hat{f}_{i,\kappa}^{(l,l)}$.
        If $z_{i,\kappa}^{(l,l)} \geq z_{(i,q)}^{(l)}$, then $f\left(z_{i,\kappa}^{(l,l)} \right) \geq f\left( z_{(i,q)}^{(l)} \right)$ by monotonicity of $f$. For the case where $z_{i,\kappa}^{(l,l)} < z_{(i,q)}^{(l)}$, define $x_{\lambda} \in \sS_{l}$ such that $z_{(i,q)}^{(l)} = z_{i,\lambda}^{(l,l)}$ and note that
        \begin{equation*}
            \hat{f}_{i,\kappa}^{(l,l)} > \hat{f}_{i,\lambda}^{(l,l)} \geq f\left( z_{i,\lambda}^{(l,l)} \right) - \tau = f\left( z_{(i,q)}^{(l)} \right) - \tau.
        \end{equation*}
        Therefore
        \begin{equation*}
            \hat{f}_{(i,q)}^{(l)} \geq f\left( z_{(i,q)}^{(l)} \right) - \tau,
        \end{equation*}
        and \eqref{eq:nfc1} holds as long as
        \begin{equation*}
            f\left( z_{(i,q)}^{(l)} \right) - \tau > f\left( \max_{k \neq l, j} z_{i,j}^{(l,k)} \right) + \tau,
        \end{equation*}
        or equivalently if
        \begin{equation}
            \tau < \frac{f\left( z_{(i,q)}^{(l)} \right) - f\left( \max_{k \neq l, j} z_{i,j}^{(l,k)} \right)}{2}.
            \label{eq:nfc2}
        \end{equation}
        Taking the minimum right-hand side of \eqref{eq:nfc2} among all $x \in \sX$ completes the proof.

\subsection{Proof of Theorem~\ref{thm:cdel}}
\label{pf:lem:cdel}

To prove Theorem~\ref{thm:cdel} from the paper, we first prove a slightly more general result that we will then apply.

\begin{lemma} \label{lem:c}
Let $a_1,\dots,a_n \in \bR^d$ be i.i.d. uniform on $\bS^{d-1}$ and let $\tilde{G}$ be the corresponding $q$-nearest neighbor graph with respect to the (transformed and noisy) inner products
\begin{equation}
\hat{f}_{ij} = f(\abs{\ip{a_i}{a_j}}) + \tau_{ij}, \quad i,j \in 1,\dots,n
\end{equation}
where
$f:\bR_+ \to \bR_{+}$ is a strictly increasing function
and $\tau_{ij} \in [-\tau,\tau]$ are bounded measurement errors.
Let $\delta \geq 0$ and $\gamma \in (1,n/\log n)$ be arbitrary,
and let $\theta$ be the spherical radius of a spherical cap covering $\gamma \log n /n$ fraction of the area of $\bS^{d-1}$.
Then if
$q \in [3(24\pi)^{d-1}\gamma \log n+3\frac{\sL(\bS^{d-2})}{\sL(\bS^{d-1})}\frac{n}{d-1}(2\delta)^{d-1},n]$, $\theta \leq (\pi/2 - \delta)/24$
and $\tau \leq \{f(\cos(16\theta)) - f(\cos(16\theta+\delta))\}/2$, we have
\begin{equation} \label{eq:c_lem_p}
\bP\{\tilde{G} \text{ is connected }\}
\geq 1 - \frac{2}{n^{\gamma-1}\gamma \log n}
,
\end{equation}
where $\sL$ denotes the Lebesgue measure of its argument.
\end{lemma}

\begin{proof}[Proof of Lemma~\ref{lem:c}]%
Following the approach taken in~\cite[Appendix A.B]{heckel2015robust}, we partition the unit sphere $\bS^{d-1}$ into $M := n/(\gamma \log n)$ non-overlapping regions $R_1,\dots,R_M$ of equal area with spherical diameters upper bounded as
\begin{equation*}
\quad \sup_{x,y \in R_m} \arccos(\langle x,y\rangle) \leq 8 \theta =: \theta^*
\end{equation*}
for all $m$; the existence of such a partition was shown in~\cite[Lemma~6.2]{leopardi2009diameter}.
Consider the events
\begin{align*}
A_m &:= R_m \text{ contains at least one of } a_1,\dots,a_n \\
B_m &:= \text{Fewer than } q/2 \text{ samples are within } 3\theta^*+\delta \text{ of } c_m \text{ in spherical distance}
\end{align*}
where $c_1,\dots,c_M$ are arbitrarily chosen points in $R_1,\dots,R_M$, respectively, and the spherical distance between two points $x$ and $y$ is $\arccos(\langle x,y\rangle)$.
The proof proceeds as in~\cite[Appendix~A.B]{heckel2015robust} by first showing that $\tilde{G}$ is connected if $A_m$ and $B_m$ hold for all $m = 1,\dots,M$.
It then follows that
\begin{equation} \label{eq:c_union}
\bP\{\tilde{G} \text{ is connected}\}
\geq \bP\{\forall m \ A_m \wedge B_m\}
\geq 1-\sum_{m=1}^M\bP\{\neg A_m\} - \sum_{m=1}^M\bP\{\neg B_m\}
\end{equation}
where $\wedge$ is conjunction, $\neg$ is negation, and the second inequality follows from a union bound.
The proof concludes by upper bounding $\bP\{\neg A_m\}$ and $\bP\{\neg B_m\}$; substituting the bounds into~\eqref{eq:c_union} yields the final result~\eqref{eq:c_lem_p}.

\textbf{Implication.}
We show that $\tilde{G}$ is connected if $A_m$ and $B_m$ hold for all $m=1,\dots,M$, by showing that all samples in neighboring regions are connected when $B_m$ holds for all $m$.
Since each region contains at least one sample when $A_m$ holds for all $m$, it then follows that any pair of samples is connected via a chain of connections through neighboring regions and so $\tilde{G}$ is connected.

Let $a_i$ and $a_\ell$ be arbitrary samples in neighboring regions $R_m$ and $R_n$.
Then $a_\ell$ is within $2\theta^*$ of $a_i$ in spherical distance and thus $\hat{f}_{i\ell} \geq \tilde{f}(2\theta^*) - \tau$, where we define $\tilde{f}(\alpha) = f(\cos(\alpha))$ for convenience and note that it is decreasing on $[0,\pi/2]$.
Any sample $a_j$ for which $\hat{f}_{ij} \geq \tilde{f}(2\theta^*)-\tau$ must satisfy
\begin{equation}
\tilde{f}(\arccos\abs{\ip{a_i}{a_j}})
= \hat{f}_{ij} - \tau_{ij}
\geq \hat{f}_{ij} - \tau
\geq \tilde{f}(2\theta^*) - 2\tau
=\tilde{f}(16\theta) - 2\tau
\geq \tilde{f}(16\theta+\delta)
= \tilde{f}(2\theta^*+\delta)
\end{equation}
and so must also satisfy $\arccos|\langle a_i,a_j\rangle| \leq 2\theta^* + \delta$ because $\tilde{f}$ is decreasing.
Namely, any such sample must be within $2\theta^*+\delta$ of either $a_i$ or $-a_i$, and must hence be within $3\theta^*+\delta$ of either $c_m$ or $c_{m'}$ where $R_{m'}$ is the region containing $-a_i$.
Under $B_m$ and $B_{m'}$, there are fewer than $q$ such samples and so all must be connected to $a_i$.
In particular, $a_\ell$ must be connected to $a_i$, and all samples in neighboring regions are connected when $B_m$ holds for all $m$.

\textbf{Upper bound on $\bP\{\neg A_m\}$.}
As in~\cite[Eqs. (27)--(28)]{heckel2015robust}, we use the fact that each sample falls outside of $R_m$ with probability $1-1/M$ since the samples are drawn uniformly from $\bS^{d-1}$ and the $M$ regions have equal area.
The samples are furthermore drawn independently, and so
\begin{equation}
\bP\{\neg A_m\}
= \left(1-\frac{1}{M}\right)^{n}
\leq e^{-n/M}
= \frac{1}{M}\frac{1}{n^{\gamma-1}\gamma \log n}
.
\end{equation}

\textbf{Upper bound on $\bP\{\neg B_m\}$.}
For convenience let $\sC_m := \{x:\arccos(\langle x,c_m\rangle) \leq 3\theta^* + \delta\}$ denote the spherical cap of spherical radius $3\theta^*+\delta$ around $c_m$, and let $N_m$ denote the number of samples in $\sC_m$.
In this notation, $B_m$ is the event that $N_m \leq q/2$.
As in~\cite[Appendix~A.B]{heckel2015robust}, we note that $N_m$ is a binomially distributed random variable with
$n$ trials
and probability $p := \sL(\sC_m)/\sL(\bS^{d-1})$,
where $\sL$ is the area (Lebesgue measure) of a set.

We begin by bounding $q/2$ below by $3np$; this will make applying a binomial tail bound more convenient. By assumption, $3\theta^*+\delta = 24\theta+\delta \leq \pi/2$ and so we can apply~\cite[Equation (5.2)]{leopardi2009diameter} as in~\cite{heckel2015robust} to bound $p$ as
\begin{equation} \label{eq:p_upper}
p := \frac{\sL(\sC_m)}{\sL(\bS^{d-1})}
\leq \frac{\sL(\bS^{d-2})}{\sL(\bS^{d-1})}\frac{(3\theta^*+\delta)^{d-1}}{d-1}
\leq \frac{1}{2}
\left(
  \frac{\sL(\bS^{d-2})}{\sL(\bS^{d-1})}\frac{(6\theta^*)^{d-1}}{d-1}
+ \frac{\sL(\bS^{d-2})}{\sL(\bS^{d-1})}\frac{(2\delta)^{d-1}}{d-1}
\right)
\end{equation}
where the second inequality follows from the convexity of $x^{d-1}$ (when $x>0$) applied to the convex combination $x = 3\theta^*+\delta =1/2(6\theta^*)+1/2(2\delta)$.
The first term can be further bounded since
\begin{equation} \label{eq:theta_bound}
\theta^* 
\leq 4\pi \left((d-1)\frac{\sL(\bS^{d-1})}{\sL(\bS^{d-2})}\frac{\gamma\log n}{n}\right)^{1/(d-1)}
\end{equation}
as in~\cite[Equation (31)]{heckel2015robust}; the proof is the same with $3(24\pi)^{d-1}$ in place of $6(12\pi)^{d-1}$.
Substituting into~\eqref{eq:p_upper} yields
\begin{equation}
p
\leq \frac{1}{2}
\left(
(24\pi)^{d-1}\frac{\gamma\log n}{n}
+ \frac{\sL(\bS^{d-2})}{\sL(\bS^{d-1})}\frac{(2\delta)^{d-1}}{d-1}
\right)
\end{equation}
and thus
\begin{equation}
3np \leq \frac{1}{2}
\left(
3(24\pi)^{d-1}\gamma \log n
+3\frac{\sL(\bS^{d-2})}{\sL(\bS^{d-1})}\frac{n}{d-1}(2\delta)^{d-1}
\right)
\leq \frac{q}{2}
.
\end{equation}
Applying the binomial tail bound~\cite[Theorem 1]{janson2002on} as done in~\cite[Equation (29)]{heckel2015robust} now yields
\begin{equation}
\bP\{\neg B_m\} = \bP\{N_m > q/2\}
\leq \bP\{N_m > 3np\} \leq e^{-np}
\leq e^{-n/M}
= \frac{1}{M}\frac{1}{n^{\gamma-1}\gamma \log n}.
\end{equation}
The last inequality holds since $R_m \subset \sC_m$ and so $p=\sL(\sC_m)/\sL(\bS^{d-1}) \geq \sL(R_m)/\sL(\bS^{d-1})=1/M$.
\end{proof}

\begin{remark}
An alternative bound on $(\alpha+\beta)^{d-1}$ could have been used in the proof of Lemma~\ref{lem:c} to shift the constants more heavily on the $\delta$ term.
For example,
\begin{equation}
(\alpha+\beta)^{d-1} \leq \lambda \left(\frac{\alpha}{\lambda}\right)^{d-1} + (1-\lambda)\left(\frac{\beta}{1-\lambda}\right)^{d-1}
\end{equation}
for any $\lambda \in (0,1)$
and taking $\lambda \approx 1$ shifts the constants heavily onto the second term.
The proof of Lemma~\ref{lem:c} uses $\lambda=1/2$.
\end{remark}

We are now prepared to prove Theorem~\ref{thm:cdel} by applying Lemma~\ref{lem:c} with a particular choice of $\delta$.

\begin{proof}[Proof of Theorem~\ref{thm:cdel}]%
Take
\begin{equation}
C_3 = \frac{f(\cos(16\theta)) - f(\cos(16\theta+\delta))}{2}
> 0,
\end{equation}
where we note that $\theta$ is implicitly a function of $n$, $d$ and $\gamma$, and we define
\begin{equation}
\delta
= \min\left\{12\pi\left(\frac{d-1}{3}\frac{\sL(\bS^{d-1})}{\sL(\bS^{d-2})} \frac{\gamma \log n}{n}\right)^{1/(d-1)},\frac{\pi}{2}-24\theta\right\}
> 0,
\end{equation}
which is also implicitly a function of $n$, $d$ and $\gamma$.
Now we need only to verify that the conditions of Theorem~\ref{thm:cdel} satisfy Lemma~\ref{lem:c}.
Note first that by construction
$\delta \leq \pi/2 - 24\theta$
and so 
$\theta \leq (\pi/2-\delta)/24$.
Furthermore
\begin{equation}
3\frac{\sL(\bS^{d-2})}{\sL(\bS^{d-1})}\frac{n}{d-1}(2\delta)^{d-1}
\leq (24\pi)^{d-1} \gamma \log n
\end{equation}
and so
\begin{align}
q &\geq 4(24\pi)^{d-1}\gamma \log n
= 3(24\pi)^{d-1}\gamma \log n + (24\pi)^{d-1}\gamma \log n \\
&\geq 3(24\pi)^{d-1}\gamma \log n+3\frac{\sL(\bS^{d-2})}{\sL(\bS^{d-1})}\frac{n}{d-1}(2\delta)^{d-1}
.
\end{align}
Hence all conditions of Lemma~\ref{lem:c} are satisfied and the conclusion follows.
\end{proof}

\subsection{Proof of Lemma~\ref{lem:noIntersection}}
    We again prove the statement for a fixed $x_{i} \in \sS_{l}$, taking a union bound to show the condition holds for all points. First define
    \begin{equation*}
        \alpha = \min_{l, \sD: \abs{\sD} \leq 2s, \norm{a} = 1} \norm{{U_{\sD}^{(l)}}^{T}U^{(l)}a}_{2},
    \end{equation*}
    and note that by the assumption of the lemma, there exists an $\eta > 0$ such that
    \begin{equation}
        \max_{k,l: k \neq l, \sD: \abs{\sD} \leq 2s} \norm{{U_{\sD}^{(k)}}^{T}U^{(l)}}_{2} = \alpha - \eta.
        \label{eq:numerator}
    \end{equation}
    Equation~\eqref{eq:numerator} implies that
    \begin{equation}
        \max_{k \neq l, j} z_{i,j}^{(l,k)} \leq \alpha - \eta
        \label{eq:zUB1}
    \end{equation}
    deterministically. Next, we show that
    \begin{equation}
        z_{(i,q)}^{(l)} \geq \alpha - \frac{\eta}{2}
        \label{eq:zLB1}
    \end{equation}
    with high probability. 
    The proof is nearly identical to \cite[Lemma 1]{heckel2015robust}. First, we have that
    \begin{eqnarray*}
        z_{i,j}^{(l,l)} &\sim& \norm{{U_{\sD}^{(l)}}^{T}U_{\sE}^{(l)}a_{i}^{(l)}}_{2} \abs{\ip{a_{i}^{(l)}}{a_{j}^{(l)}}} \\
        &\geq& \min_{l, \sD: \abs{\sD} \leq 2s, \norm{a} = 1} \norm{{U_{\sD}^{(l)}}^{T}U^{(l)}a}_{2} \abs{\ip{a_{i}^{(l)}}{a_{j}^{(l)}}},
    \end{eqnarray*}
    where the sets $\sD, \sE \subset [D]$ are the indices of the unobserved entries of $x_{j}^{(l)}$ and $x_{i}^{(l)}$, respectively. Letting $\tilde{z}_{i,j}^{(l,l)} = \abs{\ip{a_{i}^{(l)}}{a_{j}^{(l)}}}$, we see that
    \begin{eqnarray*}
        \bP \set{z_{i,j}^{(l,l)} \leq z} &\leq& \bP \set{\min_{l, \sD: \abs{\sD} \leq 2s, \norm{a} = 1} \norm{{U_{\sD}^{(l)}}^{T}U^{(l)}a}_{2}\tilde{z}_{i,j}^{(l,l)} \leq z} \\
        &=& \bP \set{\tilde{z}_{i,j}^{(l,l)} \leq \frac{z}{\alpha}}.
    \end{eqnarray*}
    We can bound the probability that \eqref{eq:zLB1} does not hold as
    \begin{eqnarray*}
        \bP \set{z_{(i,q)}^{(l)} \leq \alpha - \frac{\eta}{2}} &\leq& \bP \set{\tilde{z}_{(i,q)}^{(l)} \leq 1 - \frac{\eta}{2\alpha}} \\
        &\leq& \left( e \frac{N_{l}-1}{q-1} \right)^{q-1} p^{N_{l}-q},
    \end{eqnarray*}
    where $p = \bP \set{\tilde{z}_{i,j}^{(l,l)} \leq 1 - \frac{\eta}{2\alpha}}$. Setting $\xi = \frac{N_{l} - 1}{N_{l}^{\rho} - 1}$, we obtain
    \begin{eqnarray*}
        \bP \set{z_{j}^{(l)} \leq 1 - \frac{\eta}{2\alpha}} &\leq& \left( e\xi \right)^{\frac{N_{l}-1}{\xi}} p^{\left( N_{l}-1 \right)\left( 1 - \frac{1}{\xi} \right)} \\
        &=& \left( \left( e\xi \right)^{\frac{1}{\xi}} p^{1 - \frac{1}{\xi}} \right)^{N_{l}-1} \\
        &\leq& e^{-(N_{l}-1)c_{1}},
    \end{eqnarray*}
    where the last inequality holds for a constant $c_{1} > 0$ as long as
    \begin{equation*}
        \left( e\xi \right)^{\frac{1}{\xi}} p^{1 - \frac{1}{\xi}} < 1 \Leftrightarrow \left( e\xi \right)^{-\frac{1}{\xi - 1}} > p.
    \end{equation*}
    This inequality can be satisfied for every $p < 1$ by taking $N_{0}$, and consequently $\xi$, sufficiently large. By inspection, we have $p < 1$ as long as $\eta > 0$, which is true by assumption of the lemma.

    By monotonicity of $f$, \eqref{eq:zUB1} implies that
    \begin{equation*}
        f\left( \max_{k \neq l, j} z_{i,j}^{(l,k)} \right) \leq f\left( \alpha - \eta \right)
    \end{equation*}
    and \eqref{eq:zLB1} implies that
    \begin{equation*}
        f\left( z_{(i,q)}^{(l)} \right) \geq f\left( \alpha - \frac{\eta}{2} \right).
    \end{equation*}
    Finally, we have that
    \begin{eqnarray*}
        C_{i,l} &:=& f\left( z_{(i,q)}^{(l)} \right) - f\left( \max_{k \neq l, j} z_{i,j}^{(l,k)} \right) \\
        &\geq& f\left( \alpha - \frac{\eta}{2} \right) - f\left( \alpha - \eta \right) > 0,
    \end{eqnarray*}
    where the second line follows by monotonicity of $f$, noting that $\alpha - \eta/2 > \alpha - \eta$.
    Taking $C_{1} = \min_{l \in [K], i \in [N_{l}]} C_{i,l}/2$ and a union bound completes the proof.

\subsection{Proof of Lemma~\ref{lem:intersection}}
We again prove the statement for a fixed $x_{i} \in \sS_{l}$, with a union bound completing the proof. Let $\nu = 2/3$, $N_{l} \geq 6q$, and $c_{2} > 1/20$. 
    From \cite[Appendix C]{heckel2015robust}, we have that
    \begin{equation}
        z_{(i,q)}^{(l)} \geq \frac{\nu}{\sqrt{d_{l}}} - \varepsilon
        \label{eq:zLB2}
    \end{equation}
    and
    \begin{equation}
        \max_{k \neq l, j} z_{j}^{(k)} \leq \alpha + \varepsilon
        \label{eq:zUB2}
    \end{equation}
    with probability at least $1 - e^{-c_{2}(N_{l}-1)} - 10Ne^{-\beta^{2}/2}$,
    where
    \begin{equation*}
        \alpha = \frac{\beta(1+\beta)}{\sqrt{d_{l}}} \max_{k \neq l} \frac{1}{\sqrt{d_{k}}} \norm{{U^{(k)}}^{T}U^{(l)}}_{F},
    \end{equation*}
    \begin{equation*}
        \varepsilon = \frac{2\sigma(1+\sigma)}{\sqrt{D}}\beta
    \end{equation*}
    and $\frac{1}{\sqrt{2\pi}} \leq \beta \leq \sqrt{D}$. 
    Let $\beta = \sqrt{6\log N}$ and note that $D \geq 6 \log N$ implies $\beta \leq \sqrt{D}$. Noting that $q < N_{min}/6$ implies $N > 6$, we have $(1+\beta) < 4 \sqrt{\log N}$. These are sufficient to guarantee that $\alpha + \varepsilon <
    \frac{\nu}{\sqrt{d_{l}}} - \varepsilon$. 
    By monotonicity of $f$, \eqref{eq:zUB2} implies that
    \begin{equation*}
        f\left( \max_{k \neq l, j} z_{i,j}^{(l,k)} \right) \leq f\left( \alpha + \varepsilon \right)
    \end{equation*}
    and \eqref{eq:zLB2} implies that
    \begin{equation*}
        f\left( z_{(i,q)}^{(l)} \right) \geq f\left( \frac{\nu}{\sqrt{d_{l}}} - \varepsilon \right).
    \end{equation*}
    Finally, we have that
    \begin{eqnarray*}
        C_{i,l} &:=& f\left( z_{(i,q)}^{(l)} \right) - f\left( \max_{k \neq l, j} z_{i,j}^{(l,k)} \right) \\
        &\geq& f\left( \frac{\nu}{\sqrt{d_{l}}} - \varepsilon \right) - f\left( \alpha + \varepsilon \right) > 0,
    \end{eqnarray*}
    where the second line follows by monotonicity of $f$.
    Taking $C_{2} = \min_{l \in [K], i \in [N_{l}]} C_{i,l}/2$ and a union bound completes the proof.

\subsection{Proof of Theorem~\ref{thm:ekss0}}

By Lemma~\ref{lem:coassocProb}, the expected entries of the co-association matrix obtained by EKSS-0 are an increasing function of the inner product between points. It remains to show how tightly these values concentrate around their mean. This concentration allows us to bound the noise level $\tau$ via the following lemma.

\begin{lemma}
    \label{lem:fConcentration}
    Let $A$ be the affinity matrix formed by EKSS-0 (line 12, Alg.~\ref{alg:EKSS}).
    For two points $x_{i}, x_{j} \in \sX$, let 
    \begin{equation*}
        f_{\bar{K},\bar{d}}\left( \abs{\ip{x_{i}}{x_{j}}} \right) = \bE A_{i,j} = \bP \set{x_{i}, x_{j} \text{ co-clustered}}
    \end{equation*}
    and
    \begin{equation*}
        \hat{f}_{i,j} = A_{i,j} = \frac{1}{B} \sum_{b = 1}^{B} \ind{x_{i}, x_{j} \text{ co-clustered in } \sC^{(b)}}.
    \end{equation*}
    Then for all $\tau > 0$
    \begin{equation}
        \label{eq:fConcentration}
        \bP \set{ \abs{\hat{f}_{i,j} - f_{\bar{K},\bar{d}}\left( \abs{\ip{x_{i}}{x_{j}}} \right)} > \tau} < 2e^{-c_3 \tau^{2} B},
    \end{equation}
    where $c_3 = 2 \sqrt{\log 2}$ and the randomness is with respect to the subspaces drawn in EKSS-0 (line 4, Alg.~\ref{alg:EKSS}).
\end{lemma}
\begin{proof}
    The proof relies on sub-Gaussian concentration. The measurements $\hat{f}$ are bounded and hence sub-Gaussian with parameter $\frac{1}{\sqrt{\log 2}}$.
    Note that $\hat{f}_{i,j}$ is the empirical estimate of $f_{\bar{K},\bar{d}}\left( \abs{\ip{x_{i}}{x_{j}}} \right)$, and thus $\bE \hat{f}_{i,j} = f_{\bar{K},\bar{d}}\left( \abs{\ip{x_{i}}{x_{j}}} \right)$. Therefore, by the General form of Hoeffding's inequality \cite[Theorem 2.6.2]{vershynin2018high}
    \begin{equation*}
        \bP \set{ \abs{ \hat{f}_{i,j} - \bE \hat{f}_{i,j} } > \tau } \leq 2e^{-c_3 \tau^{2} B},
    \end{equation*}
    where $c_3 = 2 \sqrt{\log 2}$.
\end{proof}

Combining the results of Theorem \ref{lem:coassocProb} and Lemma \ref{lem:fConcentration} shows that the $(i,j)$th entry of the affinity matrix is $\tau$-\affclose/ with high probability for a single point. A union bound over all $N(N-1)/2$ unique pairs completes the proof.

\subsection{Proof of Lemma~\ref{lem:coassocProb}}

For notational compactness, we instead prove that the probability is a \emph{decreasing} function of the angle $\theta$ between points and note that $z = \cos(\theta)$.
Let $U_1,U_2,\ldots, U_K \in \bR^{D \times d}$ be the $K$ candidate bases. Let $\tilde{p}(\theta)$ be the probability that two points that are at angle $\theta$ apart are assigned to the candidate $U_1$. Then we clearly have $p_{K,D}(\theta) = K \tilde{p}(\theta)$, and it suffices to prove that $\tilde{p}$ is strictly decreasing.
        
Let $e_1,\ldots,e_D$ be the standard basis vectors in $\bR^D$. For a given $\theta$, set $x_{i} := e_1$, and $x_{j} = x_{j}(\theta) := \cos(\theta) e_1 + \sin(\theta) e_2$. By definition, for any orthogonal transformation $Q$ of $\bR^D$,
\begin{equation*}
    \tilde{p}(\theta) = \bP\set{Qx_{i}, Qx_{j} ~\textnormal{both assigned to}~U_1}.
\end{equation*}
        
We may average out this equation over a choice subgroup of orthogonal matrices. Indeed, let $L$ denote the span of $e_1$ and $e_2$, and let $Q$ be a random matrix uniformly distributed over the set of orthogonal matrices that decompose into a rotation on $L$ and the identity on $L^\perp$. We take expectations with respect to $Q$ and exchange the order of integration to get
\begin{align*}
    \tilde{p}(\theta) & = \bE_Q\bP_{U_1,\ldots,U_K}\set{Qx_{i}, Qx_{j} ~\textnormal{both assigned to}~U_1} \\
    & = \bE_{U_1,\ldots,U_K}\bP\set{Qx_{i}, Qx_{j} ~\textnormal{both assigned to}~U_1 \st U_1,\ldots, U_K}.
\end{align*}

Now fix $U_1,\ldots,U_K$. Let $A = A(\theta)$ be the event that $Qx_{i}$ and $Qx_{j}(\theta)$ are both assigned to $U_1$. We claim that $\bP\set{A(\theta) \st U_1,\ldots,U_K}$ is non-increasing in $\theta$. To see this, let us examine the event more closely. By the definition of candidate assignment, $A$ occurs when  $U_1$ is the \textit{closest} candidate to both $x_{i}$ and $x_{j}$. More mathematically, this is when 
\begin{equation} \label{eq: 1}
    \norm{P_{U_1} Qz}_2^2 > \norm{P_{U_k} Qz}_2^2, \quad \text{for} \quad 1 < k \leq K, \quad \text{and} \quad z = x_{i},x_{j}.
\end{equation}
Here, we use $P_{F}$ to denote the orthogonal projection onto a subspace $F$.
        
We shall attempt to rewrite \eqref{eq: 1} in a more useful form. First, observe that
\begin{align} \label{eq: 2}
    \norm{P_{U_1} Qz}_2^2 - \norm{P_{U_k} Qz}_2^2 & = z^TQ^TP_{U_1}^T P_{U_1}z - z^TP_{U_k}^T P_{U_k}Qz  \nonumber \\
    & = z^T Q^TP_L \paren{P_{U_1}^T P_{U_1} - P_{U_k}^T P_{U_k} }P_L^T Qz.
\end{align}
Let us also introduce some new notation. We use $\tilde{x}_{i}$ and $\tilde{x}_{j}$ to denote the two-dimensional coordinate vectors of $x_{i}$ and $x_{j}$ with respect to $e_1$ and $e_2$, we let $\tilde{Q}$ denote the restriction of $Q$ to $L$, and similarly let $\tilde{P}_L$ be the projection $P_L$ treated as a map from $\bR^D$ to $\bR^2$. We therefore have
\begin{equation*}
    z^T Q^TP_L \paren{P_{U_1}^T P_{U_1} - P_{U_k}^T P_{U_k} }P_L^T Qz = \tilde{z}^T \tilde{Q}^T M_k \tilde{Q}\tilde{z},
\end{equation*}
where $M_k := \tilde{P_L} \paren{P_{U_1}^T P_{U_1} - P_{U_k}^T P_{U_k} }\tilde{P_L}^T$. Following these calculations, we see that \eqref{eq: 1} is equivalent to
\begin{equation} \label{eq: 3}
    \tilde{z}^T \tilde{Q}^TM_k \tilde{Q}\tilde{z} > 0, \quad \text{for} \quad 1 < k \leq K, \quad \text{and} \quad \tilde{z} = \tilde{x}_{i},\tilde{x}_{j}.
\end{equation}
When $\tilde{Q}$ is fixed, denote by $A_{\tilde{Q}}$ the event over which \eqref{eq: 3} holds.

Observe that $M_k$ is a 2 by 2 real symmetric matrix. As such, the set $S_k$ of points $\tilde{z}$ in $\bR^2$ for which $\tilde{z}^TM_k \tilde{z} > 0$ comprises the union of two (possibly degenerate) antipodal \textit{sectors}. The same is true for the intersection $S := \cap_{k > 1}S_k$. Let $\phi = \phi(U_1,\ldots,U_K)$ denote the angle spanned by one of the two sectors comprising $S$, and note that $0 \leq \phi \leq \pi$. Furthermore, let $T$ be the union of the sector spanned by
$\tilde{x}_{i}$ and $\tilde{x}_{j}$ with its antipodal reflection. Then $A_{\tilde{Q}}$ holds if and only if $\tilde{Q}T \subset S$ or $S^c \subset \tilde{Q}T$. It is a simple exercise to compute
\begin{equation*}
    \bP\set{\tilde{Q}T \subset S \st U_1,\ldots, U_K} = \frac{(\phi-\theta)_+}{\pi},
\end{equation*}
\begin{equation*}
    \bP\set{S^c \subset \tilde{Q}T \st U_1,\ldots, U_K} = \frac{(\theta - \pi + \phi)_+}{\pi}.
\end{equation*}
        
Since $A$ is the disjoint union of these events, we have
\begin{equation} \label{eq: 4}
    \bP\set{A(\theta) \st U_1,\ldots,U_K} = \frac{(\phi-\theta)_+}{\pi} + \frac{(\theta - \pi + \phi)_+}{\pi}.
\end{equation}
Differentiating at any point other than the obvious discontinuities, we have
\begin{align*}
    \frac{d}{d\theta} \bP\set{A(\theta) \st U_1,\ldots,U_K} & = \frac{d}{d\theta}\frac{(\phi-\theta)_+}{\pi} + \frac{(\theta - \pi + \phi)_+}{\pi} \\
    & = - \frac{1}{\pi}1_{(0,\phi)}(\theta) + \frac{1}{\pi}1_{(\pi-\phi,\pi/2)}(\theta) \\
    & = - \frac{1}{\pi} + \frac{1}{\pi}1_{(\phi,\pi/2)}(\theta) + \frac{1}{\pi}1_{(\pi-\phi,\pi/2)}(\theta) \\
    & \leq 0.
\end{align*}
Here, the last inequality follows from the fact that either $\phi \geq \pi/2$ or $\pi - \phi > \pi/2$, thereby completing the proof of the claim. Recalling that $\tilde{p}(\theta) = \bE_{U_1,\ldots,U_K}\bP\set{A(\theta) \st U_1,\ldots,U_K}$, we have thus proved that $\tilde{p}$ is non-increasing. To see that it is strictly decreasing, simply note that $\frac{d}{d\theta} \bP\set{A(\theta) \st U_1,\ldots,U_K} < 0$ whenever $\phi(U_1,\ldots,U_K) < \pi/2$. This occurs on a set of positive measure.

\subsection{Proof of Theorem~\ref{thm:noIntersectionFull}}

By Theorem~\ref{thm:ekss0}, the co-association matrix $\bar{A}$ is $\tau$-\affclose/ with high probability.
Applying Lemma \ref{lem:noIntersection} with $s = 0$, we obtain $C_{1} > 0$ that lower bounds the separation $\phi_{q}$ defined in \eqref{eq:taubound_nfc} with high probability. Applying Theorem~\ref{thm:cdel} with $\gamma = 3$, we obtain $C_{3} > 0$ such that the components corresponding to each subspace are connected with high probability. Setting $\tau = \min\set{C_{1},C_{3}}$ in Theorem~\ref{thm:ekss0} completes the proof.

\subsection{Proof of Theorem~\ref{thm:intersectionNoiseless}}
    
By Theorem~\ref{thm:ekss0}, the co-association matrix $\bar{A}$ is $\tau$-\affclose/ with high probability.
Applying Lemma \ref{lem:intersection} with $\sigma = 0$, we obtain $C_{2} > 0$ that lower bounds the separation $\phi_q$ defined in \eqref{eq:taubound_nfc} with high probability. Applying Theorem~\ref{thm:cdel} with $\gamma = 3$, we obtain $C_{3} > 0$ such that the components corresponding to each subspace are connected with high probability. Setting $\tau = \min\set{C_{1},C_{3}}$ in Theorem~\ref{thm:ekss0} completes the proof.

\subsection{Proof of Theorem~\ref{thm:noisy}}

By Theorem~\ref{thm:ekss0}, the co-association matrix $\bar{A}$ is $\tau$-\affclose/ with high probability.
Applying Lemma \ref{lem:intersection}, we obtain $C_{2} > 0$ that lower bounds the separation $\phi_q$ defined in \eqref{eq:taubound_nfc} with high probability. Setting $\tau = \min\set{C_{1},C_{3}}$ in Theorem~\ref{thm:ekss0} completes the proof.

\subsection{Proof of Theorem~\ref{thm:noIntersectionMissing}}

By Theorem~\ref{thm:ekss0}, the co-association matrix $\bar{A}$ is $\tau$-\affclose/ with high probability.
By \cite[Lemma 4]{heckel2015robust}, the condition \eqref{eq:affCondMissing} holds with probability at least $1 - 4e^{-c_{7}D}$ as long as \eqref{eq:missCond} is satisfied. Thus, applying Lemma \ref{lem:noIntersection} with the parameters $N_{k} = n$, $d_{k} = d$ for all $k$, the result holds with the specified probability.

\rev{
\subsection{Proof of Theorem~\ref{thm:dimReducedTSC}}

Our proof leverages the fact that approximate isometries of the form \eqref{eq:approxIsom} yields affinities that are $\tau$-angle preserving. For points $x, y \in \sX$, \eqref{eq:approxIsom} combined with the identity $\ip{x}{y} = \frac{1}{4}\left( \norm{x+y}^{2} - \norm{x-y}^{2} \right)$ implies that
\begin{equation*}
    \abs{\ip{x}{y} - \ip{\Phi x}{\Phi y}} \leq \tau
\end{equation*}
with probability at least $1 - 2e^{-c_{3}\tau^{2}p}$.
Therefore, the affinity matrix formed by setting $A_{ij} = \abs{\ip{\Phi x_{i}}{\Phi x_{j}}}$ is $\tau$-angle preserving with the same probability, i.e.,
\begin{equation*}
    \abs{A_{ij} - \abs{\ip{x_{i}}{x_{j}}}}
    \leq \abs{\ip{\Phi x_{i}}{\Phi x_{j}} - \ip{x_{i}}{x_{j}}}
    \leq \tau
    .
\end{equation*}
The remainder of the proof then follows that of Thm.~\ref{thm:intersectionNoiseless}.
}

\section{Algorithmic and Simulation Details}
\label{app:simDetails}

In this section, we include implementation details beyond those included in the main body. We first provide pseudocode for the \algname{Thresh} and \algname{EKSS-0} algorithms. We then describe all preprocessing steps and parameters used for our experiments on real data.

\subsection{Pseudocode}
\label{app:pseudocode}

In Algorithm \ref{alg:Threshold} is the pseudocode for the \algname{Thresh} routine used in the EKSS algorithm, which results in the same connectivity as thresholding in TSC \cite{heckel2015robust}. Algorithm~\ref{alg:EKSS0} gives the pseudocode for the EKSS-0 algorithm, which is analyzed in Section~\ref{sec:theory}.

\begin{algorithm*}[h]
    \caption{\algname{Affinity Threshold (Thresh)}}
    \label{alg:Threshold}
    \begin{algorithmic}[1]
        \STATE \textbf{Input:} $A \in \brac{0,1}^{N \times N}$: affinity matrix, $q$: threshold parameter
        \STATE \textbf{Output:} $\bar{A} \in \brac{0,1}^{N \times N}$: thresholded affinity matrix
        \FOR{$i = 1,\dots,N$}
        \STATE $Z^\text{row}_{i,:} \gets A_{i,:}$ with the smallest $N-q$ entries set to zero. \hfill Threshold rows
        \STATE $Z^\text{col}_{:,i}\ \gets A_{:,i}$ with the smallest $N-q$ entries set to zero. \hfill Threshold columns
        \ENDFOR
        \STATE $\bar{A} \gets \frac{1}{2}\paren{Z^\text{row}+Z^\text{col}}$ \hfill Average
    \end{algorithmic}
\end{algorithm*}

\begin{algorithm}[h]
    \caption{\algname{EKSS-0}}
    \label{alg:EKSS0}
    \begin{algorithmic}[1]
        \STATE \textbf{Input:} $\sX = \set{x_{1},x_{2},\dots,x_{N}} \subset \bR^D$: data, $\bar{K}$: number of candidate subspaces, $\bar{d}$: candidate dimension, $K$: number of output clusters, $q$: threshold parameter, $B$: number of base clusterings,
        \STATE \textbf{Output:} $\sC=\set{c_1, \dots,c_K}$: clusters of $\sX$
        \FOR{$b = 1,\dots,B$ (in parallel)}
            \STATE $U_{1},\dots,U_{\bar{K}} \overset{iid}{\sim} \operatorname{Unif}(\operatorname{St}(D,\bar{d}))$  \hfill Draw $\bar{K}$ random subspace bases 
            \STATE $c_k \gets \set{x \in \sX\ :\ \ \forall j \ \norm{U_k^Tx}_2 \geq \norm{U_j^Tx}_2 }$ for $k = 1,\dots,\bar{K}$ \hfill Cluster by projection
            \STATE $\sC^{(b)} \gets \set{c_{1},\dots,c_{\bar{K}}}$
        \ENDFOR
        \STATE $A_{i,j} \gets \frac{1}{B}\abs{\set{b:x_i,x_j\text{ are co-clustered in }\sC^{(b)}}}$ for $i,j = 1,\dots,N$ \hfill Form affinity matrix
        \STATE $\bar{A} \gets \algname{Thresh}(A,q)$ \hfill Keep top $q$ entries per row/column
        \STATE $\sC \gets$ \algname{SpectralClustering}($\bar{A},K$) \hfill Final Clustering
    \end{algorithmic}
\end{algorithm}

\subsection{Clustering Error}
\label{app:clustErr}

The clustering error, which is the metric used for all experimental results, is computed by matching the true labels and the labels output by a given clustering algorithm,
\begin{equation*}
    \text{err} = \frac{100}{N} \left( 1 - \max_{\pi} \sum_{i,j} Q_{\pi(i)j}^{\text{out}}Q_{ij}^{\text{true}} \right),
\end{equation*}
where $\pi$ is a permutation of the cluster labels, and $Q^{\text{out}}$ and $Q^{\text{true}}$ are the output and ground-truth labelings of the data, respectively, where the $(i,j)$th entry is one if point $j$ belongs to cluster $i$ and is zero otherwise.

\subsection{Experiments on Benchmark Data}
\label{app:expDetails}

\begin{table}[h]
    \centering
    \begin{tabular}{ | c || c | c | c | c | c | c | c | }
        \hline 
        Dataset & $N$ & $K$ & $D$ \\
        \hline
        Hopkins-155 & 39-556 & 2-3 & 30-200 \\
        \hline
        Yale & 2432 & 38 & 2016 \\
        \hline
        COIL-20 & 1440 & 20 & 1024 \\
        \hline
        COIL-100 & 7200 & 100 & 1024 \\
        \hline
        USPS & 9298 & 10 & 256 \\
        \hline
        MNIST-10k & 10000 & 10 & 500 \\
        \hline
    \end{tabular}
    \caption{Datasets used for experiments with relevant parameters; $N$: total number of samples, $K$: number of clusters, $D$: ambient dimension.}
    \label{tab:datasets}
\end{table}

In this section, we describe the benchmark datasets used in our experiments, as well as any preprocessing steps and the parameters selected for all algorithms. All datasets are normalized so that each column lies on the unit sphere in the corresponding ambient dimension, as is common in the literature \cite{soltanolkotabi2012geometric,heckel2015robust,he2016robust}. Table \ref{tab:datasets} gives a summary of all datasets considered.

The Hopkins-155 dataset \cite{tron2007benchmark} consists of 155 motion sequences with $K = 2$ in 120 of sequences and $K = 3$ in the remaining 35. In each sequence, objects moving along different trajectories each lie near their own affine subspace of dimension at most 3. We perform no preprocessing steps on this dataset.

The Extended Yale Face Database B \cite{georghiades2001from,lee2005acquiring} consists of 64 images of each of 38 different subjects under a variety of lighting conditions. Each image is of nominal size $192 \times 168$ and is known to lie near a 9-dimensional subspace \cite{basri2003lambertian}. We downsample so that each image is of size $48 \times 42$, as in \cite{elhamifar2013sparse}. For EKSS, KSS, CoP-KSS, MKF, and TSC, we perform an initial whitening as in
\cite{zhang2012hybrid,heckel2015robust} by removing the first two singular components of the dataset and then project the data onto its first 500 principal components to reduce the computational complexity of these methods. Whitening resulted in worse performance for all other algorithms, so we omitted this step.

\begin{table*}[h]
    \resizebox{\textwidth}{!}{
    \centering
    \begin{tabular}{ | c || c | c | c | c | c | c |}
        \hline 
        Algorithm & Hopkins & Yale & COIL-20 & COIL-100 & USPS & MNIST-10k \\
        \hline
        EKSS & $d = 3, q = 2$ & $d = 2, q = 6$ & $d = 2, q = 6$ & $d = 8, q = 7$ & $d = 13, q = 3$ & $d = 13, q = 72$ \\
        \hline
        KSS & $d = 3$ & $d = 3$ & $d = 1$ & $d = 5$ & $d = 9$ & $d = 13$ \\
        \hline
        CoP-KSS & $d = 4$ & $d = 6$ & $d = 9$ & $d = 1$ & $d = 7$ & $d = 18$ \\
        \hline
        MKF & $d = 3$ & $d = 17$ & $d = 19$ & $d = 18$ & $d = 20$ & $d = 20$ \\
        \hline
        TSC & $q = 3$ & $q = 3$ & $q = 4$ & $q = 4$ & $q = 3$ & $q = 3$ \\
        \hline
        SSC-ADMM & $\rho = 0.1, \alpha = 226.67$ & $\rho = 0.1, \alpha = 670$ & $\rho = 0.8, \alpha = 5$ & $\rho = 1, \alpha = 20$ & $\rho = 1, \alpha = 20$ & $\rho = 1, \alpha = 20$ \\
        \hline
        SSC-OMP & $\varepsilon = 2^{-52}, k_{max} = 2$ & $\varepsilon = 2^{-52}, k_{max} = 2$ & $\varepsilon = 2^{-52}, k_{max} = 2$ & $\varepsilon = 2^{-52}, k_{max} = 2$ & $\varepsilon = 2^{-52}, k_{max} = 29$ & $\varepsilon = 2^{-52}, k_{max} = 17$ \\
        \hline
        EnSC & $\lambda = 0.01, \alpha = 98$ & $\lambda = 0.88, \alpha = 3$ & $\lambda = 0.99, \alpha = 3$ & $\lambda = 0.95, \alpha = 3$ & $\lambda = 0.95, \alpha = 50$ & $\lambda = 0.95, \alpha = 3$ \\
        \hline
    \end{tabular}}
    \caption{Parameters used in experiments on real datasets for all algorithms considered.}
    \label{tab:parameters}
\end{table*}

The COIL-20 \cite{nene1996coil20} and COIL-100 \cite{nene1996coil100} datasets consist of 72 images of 20 and 100 distinct objects (respectively) under a variety of rotations. All images are of size $32 \times 32$. On both datasets, we whiten by removing the first singular component when it improves algorithm performance.

The USPS dataset provided by \cite{cai11graph} contains 9,298 total handwritten digits of size $16 \times 16$ with roughly even label distribution. No preprocessing is performed on this dataset.

The MNIST dataset \cite{lecun2016mnist} contains a total of 70,000 handwritten digits, of which we consider only the 10,000 ``test'' images. The images have nominal size $29 \times 29$, and we use the output of the scattering convolutional network \cite{bruna2013invariant} of size 3,472 and then project onto the first 500 principal components as in \cite{you2016oracle}.

For all algorithms, we set $K$ to be the correct number of clusters. For EKSS, we set $B = 1000$ and $T = 3$ for all datasets except MNIST, for which we set $T = 30$. Due to the benefits demonstrated in \cite{gitlin2018improving}, we employed CoP-KSS instead of KSS as a base clustering algorithm for the USPS and MNIST datasets.
For a fair comparison to KSS, CoP-KSS, and MKF, we ran $1000$ trials of each and use the clustering result that achieves the lowest clustering error.
The parameters used for all experiments are shown in Table \ref{tab:parameters}, with the most common parameters given among the 155 datasets for the Hopkins database. For the Hopkins, Yale, and COIL-20 datasets, we performed extensive model sweeps over a wide range of values for each parameter for each algorithm. For the larger COIL-100, USPS, and MNIST-10k datasets, this was infeasible for SSC-ADMM and EnSC, so the values were instead chosen from an intelligently-selected subset of parameters.



\bibliographystyle{IEEEtran}
\bibliography{Bibliography}



\end{document}